\documentclass[dvipsnames,letterpaper,11pt]{article}

\usepackage[margin=1in]{geometry}
\usepackage[utf8]{inputenc}
\usepackage[T1]{fontenc}
\usepackage{microtype}

\usepackage{amsmath}
\usepackage{amsfonts}
\usepackage{amssymb}
\usepackage{amsthm}
\usepackage{mathtools}
\usepackage{nicefrac}
\usepackage{dsfont}

\usepackage{booktabs}
\usepackage{graphicx}
\usepackage{caption}
\usepackage{subcaption}
\usepackage{float}

\usepackage[ruled]{algorithm2e}

\SetAlFnt{\small}
\SetAlCapFnt{\small}
\SetAlCapNameFnt{\small}
\SetAlCapHSkip{0pt}
\IncMargin{-\parindent}

\usepackage{tikz}
\usepackage{tikz-network}
\usetikzlibrary{patterns,positioning}
\usetikzlibrary{arrows}
\usetikzlibrary{arrows.meta,decorations.markings}
\usetikzlibrary{patterns.meta}

\usepackage{pgfplots}
\usepackage{pgfplotstable}
\usepgfplotslibrary{external}
\pgfplotsset{compat=1.18}

\usepackage[inline]{enumitem}
\usepackage{comment}
\usepackage{cancel}
\usepackage{xcolor}
\usepackage{soul}
\usepackage{xfp}
\usepackage{xspace}
\usepackage{todonotes}

\definecolor{blue}{HTML}{34495E}
\definecolor{red}{HTML}{C0392B}
\definecolor{green}{HTML}{F1C40F}
\definecolor{grey}{HTML}{BDC3C7}

\usepackage[longnamesfirst,sort,numbers]{natbib}
\setcitestyle{numbers,comma,open={[},close={]}}

\definecolor{darkred}{RGB}{150, 0, 0}
\definecolor{darkgreen}{RGB}{0, 150, 0}
\definecolor{darkblue}{RGB}{0, 0, 150}
\usepackage[breaklinks=true,colorlinks,citecolor=darkgreen,linkcolor=darkred,urlcolor=darkblue,bookmarks=false]{hyperref}

\usepackage{cleveref}

\newcommand{\email}[1]{\protect\href{mailto:#1}{#1}}
\newcommand{\doi}[1]{%
  \href{https://doi.org/\detokenize{#1}}%
  {\url{https://doi.org/\detokenize{#1}}}%
}

\usepackage{thmtools}
\usepackage{thm-restate}

\newcommand{\reals}{\mathbb{R}}
\newcommand{\integers}{\mathbb{Z}}

\newcommand{\etal}{et~al.\@\xspace}

\newcommand{\id}{\ensuremath{\mathrm{id}}}

\newcommand{\bigtheta}{\ensuremath{\Theta}}

\newcommand{\prob}{\probability}

\newcommand{\transpose}{\intercal}

\newcommand{\symmetricdifference}{\mathrel{\triangle}}

\newcommand{\R}{\reals}
\newcommand{\tr}{\operatorname{tr}}
\newcommand{\norm}[1]{\lVert #1\rVert}
\newcommand{\eps}{\varepsilon}

\newcommand{\N}{\mathcal{N}}
\newcommand{\cP}{\mathcal{P}}
\newcommand{\cF}{\mathcal{F}}

\newcommand{\cS}{\mathcal{S}}
\newcommand{\cG}{\mathcal{G}}
\newcommand{\cA}{\mathcal{A}}
\newcommand{\cB}{\mathcal{B}}
\newcommand{\cuben}{\{-1,1\}^n}
\newcommand{\cube}{\{-1,1\}}

\DeclareMathOperator*{\argmin}{arg\,min}

\DeclareMathOperator{\expectation}{\mathbb{E}}
\DeclareMathOperator{\probability}{\mathbb{P}}
\DeclareMathOperator{\indicator}{\mathds{1}}

\DeclareMathOperator{\Unif}{Unif}

\DeclareMathOperator{\supp}{supp}

\numberwithin{equation}{section}

\theoremstyle{plain}
\newtheorem{theorem}{Theorem}[section]

\newtheorem{lemma}[theorem]{Lemma}
\newtheorem{proposition}[theorem]{Proposition}

\theoremstyle{definition}

\theoremstyle{remark}

\crefname{claim}{Claim}{Claims}
\crefname{obs}{Observation}{Observations}
\crefname{prop}{Proposition}{Propositions}
\crefname{thm}{Theorem}{Theorems}

\begin{document}

\title{Tight~$L_\infty$ Sample Complexity for Low-Degree and Sparse\\ Boolean Polynomials}

\author{
Jasper van Doornmalen\thanks{
Institute for Mathematical and Computational Engineering,
Pontificia Universidad Católica de Chile, Chile;
\email{mvandoornmalen@uc.cl}.
}
\and
Mathieu Molina\thanks{
Blavatnik School of Computer Science and AI,
Tel Aviv University, Israel;
\email{mathieum@tauex.tau.ac.il}.
}
\and
Victor Verdugo\thanks{
Institute for Mathematical and Computational Engineering, and Department of Industrial and Systems Engineering,
Pontificia Universidad Católica de Chile, Chile;
\email{victor.verdugo@uc.cl}.
}
\and
José Verschae\thanks{
Institute for Mathematical and Computational Engineering, 
Pontificia Universidad Católica de Chile, Chile;
\email{jverscha@uc.cl}.
}
}

\date{}


\maketitle

\begin{abstract}
\noindent
Motivated by the optimization of bounded binary black-box functions, we study the problem of learning polynomial surrogates over the Boolean hypercube. To ensure that optimizing the surrogate yields good solutions for the underlying objective, we require uniform $L_\infty$-error guarantees rather than the usual $L_2$-type guarantees. We characterize the minimax sample complexity of uniform estimation under subgaussian noise for two classes of bounded polynomials. First, for polynomials of degree at most $d$ on $n$ variables, the sample complexity scales as $n^{d+1}$. Second, for $s$-sparse Fourier--Walsh polynomials with $s \leq n$, it scales as $ns^2$. These rates differ structurally from the noiseless setting, where uniform exact recovery scales as $n^d$ and $ns$, respectively. Our lower bounds hold even for arbitrary adaptive learners, showing that the additional factors are intrinsic to the noisy cases. Standard Fourier-analysis tools for the $L_2$-norm do not naturally extend to the $L_\infty$-setting in a way that yields uniform guarantees. 
Our proofs overcome this difficulty by relying on suitably chosen auxiliary norms that serve as proxies for controlling the $L_\infty$-error.
Together, our results provide a tight characterization of the sample complexity of learning optimization-safe polynomial surrogates.
\end{abstract}

\section{Introduction}


Many real-world binary decision problems can be modeled as optimizing some black-box objective function~$f \colon \{-1, 1\}^n \to [-1, 1]$.
Evaluating this function might be noisy and costly, meaning that direct optimization by repeated evaluations is a practically intractable strategy.
This is particularly the case when the objective function can be evaluated only via noisy simulations, which may require substantial time. Examples include modeling the uncertain structural defects in materials~\cite{dadkhahi2020combinatorial}, seismic retrofitting and infrastructure planning under earthquake risks~\cite{jimenez2023earthquakes}, or planning of antenna drone missions to restore wireless communications in emergency situations~\cite{pijnappel2024drones}.

Any Boolean function~$f \colon \{-1, 1\}^n \to \reals$ can be expressed as a polynomial of degree~$n$ via the Fourier--Walsh expansion. If we do not impose restrictions on the class of functions, then approximately optimizing $f$ is impossible with fewer than~$2^n$ samples: Since the value at every point of the function can be arbitrary, every point needs to be observed to rule out that it is the optimum. For this reason, we study two common restricted cases: polynomials of maximum degree~$d$, and sparse polynomials, that is, those with at most~$s$ terms in the Fourier--Walsh expansion.
A natural way to optimize over such systems is to first learn a polynomial surrogate~$\smash[t]{\widehat f} \colon \{-1, 1\}^n \to \reals$ that approximates~$f$, and then to optimize over~$\smash[t]{\widehat f}$ to obtain a solution using the extensive toolbox for optimizing pseudo-Boolean functions, cf.\@~\citet[Section~13.4]{crama2011bookbooleanfunctions}. 

Most literature on statistical learning focuses on the setting where this surrogate is good in expectation, i.e., where~$\| f - \smash[t]{\widehat f} \|_{L_2} = \mathbb{E}_{X \sim \{-1, 1\}^n}[|f(X)-\smash[t]{\widehat f}(X)|^2]^{\smash{\nicefrac{1}{2}}}
$ is bounded 
with high probability.
Recently, in a breakthrough result, Eskenazis and Ivanisvili~\cite{eskenazis2022lowdegree} showed that, for constant maximal degree, the number of samples required for a high-probability bound is logarithmic in the number of variables.
However, optimization on surrogates~$\smash[t]{\widehat f}$ with low~$L_2$ error may exploit larger errors in specific regions of the domain, rendering the approach useless.%
\footnote{
For a constant~$\varepsilon > 0$, let~$f(x) = (1/n) \sum_{i=1}^{n} x_i$  and~$\smash[t]{\widehat f}(x) = f(x) - (\varepsilon/\sqrt{n}) \sum_{i=1}^{n} x_i$ where~$\| f - \smash[t]{\widehat f}\|_{L_2} = \varepsilon$.
For large $n$, we have that $\arg\max_{x\in \{-1,1\}^n} {\widehat f}(x)=\smash[t]{\widehat x} = (-1, \dots, -1)$, which yields the worst-possible outcome~$f(\widehat x) = -1$.
}
We thus require that the surrogate error is uniformly bounded over the entire hypercube, $\|f-\widehat f\|_{L_\infty} \leq \varepsilon$, so that the value $f(\widehat x)$ of any optimal solution $\widehat x$ for $\smash[t]{\widehat f}$ is no more than $2\varepsilon$ away from the optimal value for~$f$. Hence, we focus on high-probability bounds estimating~$f$ within an~$\varepsilon$ tolerance with respect to the $L_\infty$ norm. This is particularly desirable when~$f$ must be optimized multiple times subject to different constraints. 

\paragraph{Our contributions.} \label{par:contributions}
We study two classes of binary functions~$f \colon \{-1, 1\}^n \to [-1, 1]$:
\begin{enumerate*}[label={\itshape(\roman*)}, ref={\itshape(\roman*)}]
    \item $\mathcal{P}_{n, d}$, the class of low-degree polynomials of maximal degree~$d$, and
    \item $\mathcal{S}_{n, s}$, the class of~$s$-sparse polynomials, i.e.\@ where the functions have at most~$s$ terms.
\end{enumerate*}
These functions can be queried under additive independent and identically distributed $\sigma$-subgaussian noise. 
We characterize the sample complexity of recovering \(f\) uniformly over the Boolean hypercube, up to additive error \(\varepsilon\) at every point, with probability at least \(1-\delta\).
Crucially, our bounds are tight for constant degree~$d$ and noise bounded from below. 
Altogether, our results characterize the sample complexity of learning optimization-safe surrogates from noisy observations. We also characterize the $L_\infty$ sample complexity of those two classes up to polylog factors in the noiseless setting.

More precisely, in the case with noise ($\sigma \geq  \sigma_0>0$), we show for the low-degree setting~$\mathcal{P}_{n, d}$ for constant~$d$, and for the highly sparse setting~$\mathcal{S}_{n, s}$ for~$1 \leq s \leq n$, that the sample complexity is
\begin{align*}
m^*(\mathcal{P}_{n, d}, \varepsilon, \delta) = \bigtheta_d \big( \frac{\sigma^2}{\eps^2} n^d(n + \log(\frac{1}{\delta}))\big)
\ \text{and}\ 
m^*(\mathcal{S}_{n, s}, \varepsilon, \delta) = \bigtheta\big( \frac{\sigma^2}{\eps^2}(n s^2 + s \log (\frac{1}{\delta}))\big).
\end{align*}
In the noiseless cases ($\sigma = 0$), these figures are respectively~$\bigtheta_d(n^{d})$ and~$\tilde{\bigtheta}(sn)$.
For the low-degree case, we prove a probabilistic lower bound on the injective norm associated with the $\Vert \cdot \Vert_\infty$ norm of a Gaussian tensor, which yields a Bayes risk lower bound combined with the Bohnenblust--Hille inequality~\cite{defant2018fourier,bohnenblust1931bound}. For the sparse case, we obtain a sharp dependency on the high confidence term $\log(1/\delta)$ for the upper bound, controlling the $\Vert \cdot \Vert_1$ coefficient norm via the top-$s$ norm, which obtains a better rate than a direct application of classical Lasso results for the $\Vert \cdot \Vert_2$ norm (see e.g. \cite[Chapter 7]{Wainwright_2019}). For the lower bound, we exhibit a hard family of polynomials for which the $L_\infty$ norm is equivalent to the Hamming distance on the hypercube, which, together with Fano's inequality, achieves the claimed rate.

\paragraph{Related work.}
The Fourier--Walsh representation for functions over Boolean domains is a commonly used approach. This technique is used in the papers by Linial \etal~\cite{linial1993fourier} and Eskenazis and Ivanisvili~\cite{eskenazis2022lowdegree}. 
The latter show that, for constant degree, a number of uniformly random samples logarithmic in~$n$ suffices to learn low-degree functions on the hypercube with high probability in the~$L_2$ norm, and this is later proven to be tight~\cite{eskenazis2024metricentropy}. In spirit, both works rely on identifying (some) terms of the degree-truncated Fourier--Walsh representation via sampling. Eskenazis and Ivanisvili refine this approach by distinguishing between large and small coefficients, and applying Bohnenblust--Hille-type hypercontractive inequalities~\cite{bohnenblust1931bound,defant2018fourier} to the large coefficients, which enables their logarithmic sample complexity bound. 
Various authors studied learning in the low-degree setting~\cite{andoni2014lowdegree}, and in the sparse setting~\cite{andoni2014sparse,negahban20212,stobbe2012,amrollahi2019} on the hypercube, even under a noise model.
Our work is also related to high-dimensional minimax estimation and sparse regression~\cite{tsybakov2008introduction,Wainwright_2019}. In the Walsh basis, the sparse case is a special case of a sparse linear model in dimension \(2^n\), connecting it to the Lasso, Dantzig selector, and compressed sensing literature~\cite{tibshirani1996,Candes2007,bickel2009}. The connection is not direct, however: Walsh--Hadamard rows are not generic random designs, and our target loss is uniform \(L_\infty\) recovery rather than prediction or coefficient error.

Regarding optimization, once a polynomial representation of a Boolean function is present, one can apply the toolbox of pseudo-Boolean optimization and mixed-integer nonlinear optimization. Among others, this includes linearization, convexification and branch-and-bound methods. We refer to Crama and Hammer~\cite[Sec.\@~13.4]{crama2011bookbooleanfunctions} for an overview of pseudoboolean optimization techniques, and Lee and Leyffer~\cite{lee2012minlp} for wider mixed-integer nonlinear optimization methods.
Dadkhahi \etal~\cite{dadkhahi2020combinatorial}, Dadkhahi \etal~\cite{dadkhahi2022blackbox} and Derbel \etal~\cite{derbel2023pseudobooleansurrogate} show promising empirical results in applying polynomial surrogates for optimization: The first two papers learn a surrogate first and then optimize (which is closer to our goal), and the third paper heuristically reinforces Fourier--Walsh-based surrogates within the optimization. In all cases, these authors provide promising empirical results, but theoretical guarantees on optimality remain unanswered, which we address in this paper.

Bayesian Optimization (BO)~\cite{shahriari2016bayesianoptimizationsurvey,eriksson2021bayesian} is another surrogate-based method for black-box optimization that aims to identify a maximizer by minimizing regret. Typically, BO is designed to use a small number of evaluations and guarantees sublinear bounds on cumulative regret, and therefore does not require the surrogate to be uniformly accurate across the entire domain. Bayesian Optimization of Combinatorial Structures (BOCS), surveyed by Baptista~\cite{baptista2018bocs}, and in particular the COMBO framework~\cite{oh2019bayesiancombo} extend BO to binary domains. In particular, BOCS models the objective using low-degree (typically quadratic) polynomial surrogates, reflecting the common focus on fixed, small degrees in the literature. BOCS iteratively selects the next evaluation points by solving a combinatorial optimization problem. 
In contrast, we study conditions under which the surrogate itself is uniformly accurate, enabling the use of the same surrogate across multiple optimization tasks or under initially unknown constraints. 
No prior work provides tight sample complexity guarantees for $L_\infty$-accurate learning of these function classes under noise.


\vspace{-0.2cm}

\section{Model}\label{sec:model}

\vspace{-0.3cm}

We now fix the notation and the precise query model used in the low-degree and sparse analyses.  The model is deliberately stated for arbitrary adaptive learners: lower bounds therefore rule out any strategy that chooses queries based on previous noisy evaluations, while the upper bounds we give later are nonadaptive.

\vspace{-0.2cm}

\paragraph{Fourier--Walsh representation.} We consider bounded functions $f:\cuben\to[-1,1]$ on the $n$-dimensional Boolean hypercube. For $1\le p<\infty$, we use the normalized Boolean $L_p$ norm $\norm{f}_{L_p}\coloneqq\left(\expectation[\vert f(X)\vert^p]\right)^{1/p}$, where $X\sim\mathrm{Unif}(\cuben)$. We also define the $L_\infty$ norm as $\norm{f}_{L_\infty}\coloneqq\max_{x\in\cuben}\vert f(x)\vert$. 
For $S\subseteq[n]$, we denote the associated Walsh function as $\chi_S(x)\coloneqq\prod_{i\in S}x_i$ for all $x \in \cuben$, with $\chi_\emptyset=1$. The family $\chi=(\chi_S)_{S\subseteq[n]}$ is an orthonormal basis of $L_2(\cuben)$ under the uniform measure, as $\expectation[\chi_S(X)\chi_T(X)]=\mathds{1}\{S=T\}$. Hence, every function $f:\cuben\to [-1,1]$ admits a unique Fourier--Walsh expansion:
$f(x)={\sum\nolimits_{S\subseteq[n]}\theta_S(f)\chi_S(x)}$ for~$x\in\cuben$,
with~$\theta_S(f)\coloneqq\expectation[f(X)\chi_S(X)]$ as the Fourier coefficients~\citep{odonnell2014analysisbooleanfunctions}. We denote the coefficients vector by~$\theta(f)\coloneqq(\theta_S(f))_{S\subseteq[n]}$. 

Conversely, for a vector~$\theta\in\R^{2^n}$, we write $f_\theta$ for the function which has coefficients $\theta$ on the Walsh basis. When the meaning is clear from the context, we suppress this dependence and simply write $f=f_\theta$ and $\theta=\theta(f)$. It will be useful to view evaluations on a point $x \in \cuben$ as linear measurements in coefficient space $f_\theta(x)= \langle \theta , \chi(x) \rangle$, where $\chi(x)=(\chi_S(x))_{S \subseteq [n]}$ is the evaluation vector on each basis function. Thus, learning a general function on the cube can be written as a special case of a linear regression problem in dimension $2^n$. However, one crucial distinction is that the available design vectors are not arbitrary elements of $\cube^{2^n}$, but Walsh--Hadamard rows $\chi(x)$ induced by points $x\in\cuben$.

\vspace{-0.2cm}

\paragraph{Structured class of Boolean functions.} 
For $d\in \{0,1,\ldots,n\}$, we define~$\mathcal{P}_{n, d}$ as the class of low-degree polynomials with degree at most $d$, and for any non-negative~$s \in \integers_{\geq 0}$, we define~$\mathcal{S}_{n, s}$ as the class of sparse polynomials with at most $s$ non-zero Fourier coefficients. More precisely, we have
\begin{align}
 \cP_{n,d} &\coloneqq 
 \Big\{ 
    \smash{f= \sum\nolimits_{S\subseteq [n]:\vert S \vert \leq d} \theta_S \chi_S : \norm f_{L_\infty} \leq 1}
 \Big\}, \quad \text{and}
 \\
\cS_{n,s} &\coloneqq 
\Big\{ 
    \smash{
    f= \sum\nolimits_{S \subseteq [n]} \theta_S \chi_S 
    :
    \big|\{ S \subseteq [n] : \theta_S \neq 0 \}\big|
    \leq s, \norm f_{L_\infty} \leq 1
    }
    \Big\}
    .
\end{align}
In the low-degree setting, given the boundedness on the degree~$d$ of the polynomial, there are at most~$D \coloneqq \sum_{k=0}^d \binom{n}{k} \approx n^d/d!$ non-zero coefficients. 

%

\paragraph{Adaptive noisy query model.} 

We aim to recover an unknown target function $f$ from $m$ noisy query evaluations, where each query point may be chosen adaptively. Formally, at each round $t\in[m]$, the learner chooses a query point $X_t\in\cuben$, possibly depending on $\mathcal{H}_t=(X_1,Y_1,\ldots,X_{t-1},Y_{t-1})$ the previous observations up to time $t$, and observes~%
$
Y_t=f(X_t)+\xi_t
$,
where the $\xi_t$ are independent centered $\sigma$-subgaussian random variables, i.e. $\expectation[\xi_t]=0$ and $\prob(\vert \xi_t \vert \geq u) \leq 2 \exp(-u^2/2\sigma^2)$. Thus, after $m$ queries, the learner's output $\smash[t]{\widehat f}$ is a function of the full transcript $\mathcal H_m$, while the transcript itself is generated by adaptive query rules. We denote by $\smash[t]{\widehat \theta}=\theta(\smash[t]{\widehat f})$.

For an approximation error $\eps>0$ and a failure probability $\delta\in(0,1)$,
an adaptive learning algorithm succeeds on a class $\mathcal F$ with $m$ queries
if, for every target function $f\in\mathcal F$, it outputs a hypothesis
$\smash[t]{\widehat f}$ satisfying $\prob_f(\Vert f - \smash[t]{\widehat f} \Vert_{L_\infty}>\eps)\le \delta$, where $\prob_f$ highlights that the data $Y_t$ is generated via $f$. We define the adaptive sample complexity of
$\mathcal F$ by
\begin{equation*}
m^*(\mathcal F,\eps,\delta)
\coloneqq
\inf\big\{
m\ge 0:
\exists\ \text{$m$-query adaptive learner,}\ \sup\nolimits_{f\in\mathcal F, \xi \, \sigma\text{-SG}}
\prob\nolimits_f(
\Vert f-\smash[t]{\widehat f}\Vert_{L_\infty} \! >\eps
)
\le \delta
\big\}.
\end{equation*}

\paragraph{On the specificity of the $L_\infty$ norm.} 
\label{par:specificity}
Most existing guarantees for learning low-degree or sparse Boolean functions are formulated in $L_2$, often in the noiseless or random-query setting; see, e.g., \citep{odonnell2014analysisbooleanfunctions,linial1993fourier,eskenazis2022lowdegree,andoni2014lowdegree,andoni2014sparse}. In the $L_2$ norm case, due to Parseval's identity, $\Vert f - \smash[t]{\widehat f} \Vert_{L_2}= \Vert \theta - \smash[t]{\widehat \theta} \Vert_2$. Thus, controlling prediction error in $L_2$ is exactly the same as controlling coefficient error in $\Vert \cdot \Vert_2$. Moreover, for degree-at-most-$d$ polynomials, Bonami--Beckner hypercontractivity~\cite{Bonami70,Beckner75,odonnell2014analysisbooleanfunctions} implies that $\norm{f}_{L_p}\le (p-1)^{d/2}\norm{f}_{L_2}$ for every finite $p\ge 2$. Consequently, if both $f$ and $\smash[t]{\widehat f}$ have degree at most $d$, then an $L_2$ guarantee for $f- \smash[t]{\widehat f}$ automatically yields an $L_p$ guarantee for every finite $p$, with a loss independent of the ambient dimension $n$. This route fundamentally breaks at $p=\infty$: there is no such dimension-free inequality for the $L_\infty$ norm, even for degree-one functions. Similarly, the norm inequality $\Vert f - \smash[t]{\widehat f} \Vert_{L_\infty}= \sup_{x \in \cuben} \vert \langle \theta - \smash[t]{\widehat \theta}, \chi(x) \rangle \vert \leq \Vert  \theta - \smash[t]{\widehat \theta}\Vert_1$, shows that an upper bound on the $\Vert \cdot \Vert_1$ Fourier coefficient error yields an upper bound in the $L_\infty$ norm. However, equality does not hold in general. This inequality should not be confused with the usual duality between $\Vert \cdot \Vert_1$ and $\| \cdot \|_\infty$ in~$\R^{2^n}$. Indeed, in the Walsh setting, the supremum is taken only over the Walsh--Hadamard vectors $\{\chi(x):x\in\cuben\}$, rather than over all sign vectors in $\{-1,1\}^{2^{[n]}}$. Equivalently, the $L_\infty$ loss is the support function of a highly structured subset of the full sign cube. In particular, for $s$-sparse functions, we have $\Vert \theta - \smash[t]{\widehat \theta} \Vert_1 \leq \sqrt{2s} \Vert f - \smash[t]{\widehat f} \Vert_{L_\infty}$ using Cauchy--Schwarz and Parseval,\footnote{Let $\mathcal{A} \subseteq 2^{[n]}$ be the sparse support, $\Vert \theta \Vert_1=\sum_{S \in \mathcal{A}} \vert \theta_S\vert \leq \sqrt{\sum_{S \in \mathcal{A}} 1} \sqrt{\sum_{S \in \mathcal{A}} \theta_S^2}=\sqrt{\vert \mathcal{A} \vert} \Vert f \Vert_{L_2} \leq \sqrt{s} \Vert f \Vert_{L_{\infty}}$.} and this inequality is tight for all error functions which are bent polynomials~\cite[Sec.\@~6.3]{odonnell2014analysisbooleanfunctions}. These observations explain why the $L_\infty$ problem requires significantly different techniques relative to the existing $L_2$ theory. Sharp sample-complexity rates require auxiliary norms and constructions tailored to the specific $L_\infty$ geometry, particularly for lower bounds, rather than black-box applications of standard $L_2$ or generic sparse linear-regression guarantees.

\paragraph{Noiseless recovery.}
Before turning to noisy observations, let us first provide the sample complexity when $\sigma=0$ as a comparison point. We write $m^{(0)}(\mathcal F,\eps,\delta)$ for the corresponding noiseless sample complexity. This benchmark already separates the $L_\infty$ problem from the usual exact $L_2$ learning setting: for bounded low-degree functions, prior work shows that logarithmically many random queries can suffice for approximate recovery in $L_2$ \citep{eskenazis2022lowdegree}, whereas uniform recovery requires interpolation of the whole degree-$d$ space. More precisely, for deterministic algorithms we have a sample complexity of exactly $D$, and for randomized algorithms  $D/2\leq m^{(0)}(\cP_{n,d},\eps,\delta)\leq D =\Theta_d(n^d)$, which is a strictly smaller order compared to $n^{d+1}$. See \Cref{app:noiseless-recovery-low} for the proof. For the noiseless sparse case, there exist universal constants
$c,C>0$ such that $c\,sn\, \smash{\log^{-1}(2s)}
\le
m^{\smash{(0)}}(\cS_{n,s},\eps,\delta)
\le
C\,sn\, \smash{\log^2(2s)}$.
See \Cref{app:noiseless-recovery-sparse} for the proof. Thus, noiseless sparse recovery has complexity ${\widetilde\Theta(sn)}$. 

The rest of the paper studies the noisy regime. We prove that additive noise genuinely changes the scaling of uniform recovery. For fixed degree $d$, the low-degree stochastic term grows from the noiseless scale $n^d$ to order $n^{d+1}$. For sparse polynomials, the corresponding term grows from the noiseless scale $\smash{\widetilde\Theta(sn)}$ to $\Theta(ns^2)$. The lower bounds show that these additional factors arise from a statistical barrier; no algorithm can remove them, adaptive or not.

\section{Low-degree Walsh polynomials}\label{sec:low-degree}

In this section, we focus on the sample-complexity bounds for low-degree polynomials evaluated with noise. Let $B_d$ denote the degree-$d$ Boolean Bohnenblust--Hille constant~\cite{defant2018fourier,bohnenblust1931bound}, namely the
smallest constant such that every degree-at-most-$d$ Walsh polynomial $f$
satisfies~%
$
\norm{\theta(f)}_{\frac{2d}{d+1}}\le B_d\norm{f}_{L_\infty}
$.
This constant shows that the norm of the coefficients can be controlled by the $L_\infty$ norm up to a constant that depends only on $d$, and not on the dimension. While the exact dependence of $B_d$ on $d$ is unknown, the best current upper bound is $B_d \leq \exp(O(\sqrt{d\log(d)})$ \cite{defant2018fourier}. 
Recalling that~$D = \bigtheta(n^d)$ is the size of the monomial basis in $\mathcal{P}_{n,d}$, we prove the following tight sample complexity bounds.

\begin{theorem}[Noisy low-degree sample complexity]\label{thm:low-degree-main}
There exist universal constants $c,C>0$, such that for every $n\ge d\ge1$,\ \  $0<\delta\le1/4$,\ \ $0<\eps\le c/(B_d\sqrt d)$, and centered~$\sigma$-subgaussian noise:
\begin{equation}\label{eq:low-degree-upper}
\frac{c \sigma^2}{d \eps^2 (2e)^d B_d^2}
D\bigl(n+\log\big(\frac{1}{\delta}\big)\bigr)
\le
m^*(\cP_{n,d},\eps,\delta)
\le
C\left[
D\log(2D)
+
\frac{\sigma^2}{\eps^2}D\bigl(n+\log\big(\frac{1}{\delta}\big)\bigr)
\right].
\end{equation}
For fixed $d$, this implies that $m^*(\cP_{n,d},\eps,\delta)=\Theta_d(\sigma^2n^d(n+\log(1/\delta)) / \eps^2)$.
\end{theorem}

We prove the upper and lower bounds separately. The upper bound is quite standard: it is a fixed-design least-squares argument using a well-conditioned set of truncated Walsh rows. The main novelty is the lower bound that uses a Bayes-risk argument \citep{Chen2014OnBR} on a block-multilinear subclass of~$\cP_{n,d}$. The $n^{d+1}$ and $n\log(1/\delta)$ regimes are treated separately.

\subsection{Low degree upper bound}


Given the bound $d$ on the degree, for any $f \in \cP_{n,d}$, we have $\theta_S=0$ for any $\vert S \vert >d$. For $x\in\cuben$, we write 
$\chi_{\le d}(x)\coloneqq(\chi_S(x))_{\vert S \vert \leq d}$ be the $D$ dimensional vector containing the evaluation of the Walsh functions associated to the non-zero coefficients. In particular, for an $f$ represented by $\theta \in \R^D$, we have $f(x)=\langle \theta, \chi_{\leq d} (x)\rangle$, hence this is a special case of a $D$-dimensional linear regression problem. The upper bound is a fixed-design linear-regression argument with a uniform prediction guarantee over the cube. The only point that is specific to the Boolean setting is the design: the available feature vectors are the Walsh--Hadamard rows $\chi_{\le d}(x)$, not arbitrary vectors in~$\R^D$.  Nevertheless, the Walsh design behaves here like a standard bounded orthonormal system. Under the uniform measure, as $\expectation[\chi_{\le d}(X)\chi_{\le d}(X)^\top]=I_D$ is the identity matrix and $\norm{\chi_{\le d}(X)}_2^2=D$, matrix Chernoff yields a deterministic set of \(O(D\log D)\) cube points whose empirical Gram matrix is well conditioned. Once such a design is fixed, the problem reduces to ordinary least squares in dimension \(D\).
For symmetric matrices $A,B$, we write $A \preceq B$ if $B-A$ is positive semidefinite.

\begin{lemma}[Walsh design]\label{lem:walsh-design} 
There exists a universal constant $C$ and $m\ge CD\log(2D)$ points $x_1,\ldots,x_m\in\cuben$  such that
\begin{equation*}
  \frac12 I_D
  \preceq
  \frac1m\sum\nolimits_{t=1}^m
  \chi_{\le d}(x_t)\chi_{\le d}(x_t)^\top.
\end{equation*}
\end{lemma}

Consider now the upper bound in~\Cref{thm:low-degree-main}. We remark that, in contrast to the lower bound, this bound is valid for any $\varepsilon>0$ and $\delta \in (0,1)$. Using this well-conditioned design, standard linear regression analysis yields the claimed result. See \Cref{app:walsh-design} for the full proof.

We highlight that while the design is fixed, it is non-adaptive. If we had used a uniform random design instead, we would have paid $D n\log(1/\delta)$ instead of $D \log(2D)$ for the design term, which would correspond to an asymptotic growth of~$\bigtheta_d(n^{d+1} \log(1/\delta))$ in the design term in~\eqref{eq:low-degree-upper}. 

\subsection{Low degree lower bound} 

The lower bound is the main point of the section.  It shows that the additional factor $n$ in the noisy $L_\infty$ rate cannot be avoided: even an adaptive learner must leave enough posterior uncertainty to create a large error at some point of the cube. We suppose for this whole section that $\xi_t \sim \mathcal{N}(0,\sigma^2)$.

The proof uses a Bayes-risk argument. We place a Gaussian
prior on a block-multilinear degree-\(d\) homogeneous polynomial subclass. Equivalently, we draw a Gaussian \(d\)-tensor \(G\) and use it to define a suitably scaled polynomial \(f_{\mu G}\), which belongs to \(\cP_{n,d}\) with high probability. We then analyse the posterior distribution of \(G\) after an adaptive sequence of noisy queries, which itself is Gaussian by conjugacy. We then show that the covariance matrix of the Gaussian posterior is sufficiently large in two ways: for small $m$, it has a large Boolean injective norm, as well as a large trace.  This residual uncertainty is converted into an \(L_\infty\)-error for the constant confidence error through a posterior-diameter argument, using the Boolean-injective-norm lower bound, and through one-point posterior-variance in the high-confidence regime, using the trace lower bound. 

Let us now formally define the prior we will use for the lower bound. We will often go back and forth between the tensor formalism and the vectorized version. Let $q=\lfloor n/d\rfloor$. We use only the first $dq$ coordinates, split into $d$ blocks $I_1,\ldots,I_d$ each of size $q$.  Write a point in the active cube as $x=(x^{(1)},\ldots,x^{(d)})\in(\cube^q)^d$.  The hard prior lives on the block-multilinear family $\cB\coloneqq  \left\{\{i_1,\ldots,i_d\}: i_r\in I_r\text{ for each }r\right\}$, such that  $b \coloneqq |\cB|=q^d$. Equivalently, after indexing each element of $\cB$ by $(i_1,\ldots,i_d)\in[q]^d$, the corresponding Walsh feature vector is
\[
  \chi_{\cB}(x)=
  \Big(\prod\nolimits_{r=1}^d x^{(r)}_{i_r}\Big)_{(i_1,\ldots,i_d)\in[q]^d}
  \in\cube^{q^d}.
\]
Thus, $\chi_{\cB}(x)$ is the tensor product
$x^{(1)}\otimes\cdots\otimes x^{(d)}$ written in Walsh notation.  Let
$G\sim\N(0,I_{q^d})$ and define the random polynomial
\begin{equation}\label{eq:gaussian-prior}
  f_{\mu G}(x)
  =
  \mu\langle G,\chi_{\cB}(x)\rangle,
  \quad \mbox{where} \
  \mu
  = C_0B_d \eps\big(q^{d/2}\sqrt{q+\log(1/\delta)}\,\big)^{\smash{-1}}
\end{equation}
is a normalization parameter, for $C_0$ a large universal constant.  The prior is calibrated to satisfy two competing requirements.  On the one hand, after conditioning, the sampled functions must belong to the bounded class $\cP_{n,d}$.  On the other hand, if too few noisy queries are made, the posterior must still fluctuate sufficiently, of order $\eps$ in uniform norm, to make the true posterior indistinguishable from other candidates.  The Gaussian prior is not automatically in $\cP_{n,d}$ because its supremum norm is random. Nevertheless, this only happens with small probability. 
\begin{lemma} \label{lem:bounded-prior}

There exists a universal constant \(c_b>0\) such that, if
\(\eps \le c_b/(B_d\sqrt{d})\), then the Gaussian prior satisfies the
boundedness constraint with high probability, namely $\prob\left(\Vert f_{\mu G}\Vert_{L_\infty}\le 1\right)
  \ge 1-\frac{\delta}{16}$.

\end{lemma}
This is immediate due to the concentration of $\langle G, \chi_\cB(x)\rangle$ for all $x \in \cuben$ and by the union bound over all points of the Boolean cube; see \Cref{app:bounded-prior} for details. We first prove a Bayes lower bound for the unconditioned Gaussian prior and then
condition on the bounded event.


We now want to show that if $m$ is too small, then different norms and measures of the posterior distribution are large.
The next lemma is the adaptive core of the argument.  It says that below the claimed sample size, no adaptive choice of query points can remove a constant fraction of the posterior trace.   Once we condition on the
realized transcript, the adaptively chosen query points are fixed vectors, and the calculation reduces to a deterministic statement about the rank-one information matrix accumulated by the learner.

\begin{lemma}[Adaptive posterior lower bound]\label{lem:posterior-trace-main}
The posterior law of $G$ given  $\mathcal{H}_m$ is Gaussian with covariance $\Sigma_m
  =
  \big(
  I+
  \frac{\mu^2}{\sigma^2}
  \sum_{t=1}^m
  \chi_{\cB}(X_t)\chi_{\cB}(X_t)^\top
  \big)^{-1}$. If $m\le \sigma^2q^d(q+\log(1/\delta))/(C_0^2B_d^2\eps^2)$, then, for every realized transcript, we have $ 0\preceq \Sigma_m\preceq I$ and $ \tr(\Sigma_m)\ge q^d/2$. 
\end{lemma}

\begin{proof}[Proof sketch]
Conditioning on the realized transcript makes the adaptive design fixed. The
posterior covariance then follows from standard Gaussian conjugacy. The only
estimate needed is that each Walsh feature vector in the hard subspace has norm
$q^{d/2}$, so below the stated sample size the posterior precision has trace at
most a constant multiple of its prior trace. Cauchy--Schwarz on the eigenvalues
then converts this into the claimed lower bound on $\operatorname{tr}(\Sigma_m)$. See the full proof in \Cref{app:posterior-trace-main}.
\end{proof}

For a $d$-tensor $W=(W_{i_1,\ldots,i_d})$ of shape $q\times\cdots\times q$, we 
define its injective norm $\Vert W \Vert_{\mathrm{inj}}\coloneqq \sup_{\Vert x_i \Vert_\infty \leq 1, i \in [d]}\vert \sum_{i_1,\dots,i_d}  
  W_{i_1,\dots,i_d}x_{1,i_1}\cdots x_{d,i_d} \vert$. Since the supremum is attained on the Boolean cube (which follows by simply optimizing the expression one variable at a time), it is equal to the $L_\infty$ norm of the $q$-homogeneous polynomial associated to it, $ \Vert W \Vert_{\mathrm{inj}} =\Vert f_W \Vert_{L_\infty}$. For Gaussian $W$, the expectation $\expectation[\Vert W \Vert_{\mathrm{inj}}]$ can also be seen as the Gaussian width  of the stretched Boolean cube.

To prove the next lower bound on $\Vert W \Vert_{\mathrm{inj}}$, we use the Bohnenblust--Hille inequality  to convert a lower bound on a coefficient norm into a lower bound on the uniform norm of the associated homogeneous Walsh polynomial. The following lemma is the main geometric tool for the posterior-diameter
argument. Crucially, the proof does not assume that $\Sigma$ is diagonal, which means that it is valid even with correlated Gaussians, which is essential to handle adaptive queries. 

\begin{lemma}[Injective norm lower bound]\label{lem:width-main}
Let $W\sim\N(0,\Sigma)$ be a centered Gaussian $d$-tensor of shape
$q\times\cdots\times q$.  If $0\preceq\Sigma\preceq2I$ and
$\tr(\Sigma)\ge q^d$, then $\prob(
  \norm{W}_{\mathrm{inj}}
  \ge
  \frac{c}{B_d}q^{(d+1)/2}
  )
  \ge
  \frac34$ for some universal constant $c>0$.
\end{lemma}

\begin{proof}[Proof sketch]
The trace assumption implies that a constant fraction of the tensor coordinates has variance bounded below. Hence,  H\"older's inequality for \(r=2d/(d+1)\) implies that the expected
$\| \cdot \|_r$-norm of \(W\) is of order \(q^{(d+1)/2}\). Since \(W\) is a Gaussian
vector with eigenvalues bounded by $2$, and~$\Sigma \preceq 2I_b$, we can apply the concentration results for Lipschitz functions of Gaussian random variables for $\Vert W \Vert_r$, which thus concentrates around its mean at the relevant scale. Therefore,
with constant probability, \(\|W\|_r \gtrsim q^{(d+1)/2}\). The Boolean
Bohnenblust--Hille inequality applied to the associated homogeneous Boolean polynomial $f_W$ can then be used to lift the lower bound on $\Vert W \Vert_r$ to $\Vert W \Vert_{\mathrm{inj}}=\Vert f_W \Vert_{L_\infty}$,  yielding the claimed lower bound. See \Cref{app:width-main} for the proof. 
\end{proof}

\vspace{-0.2cm}


We can now prove the lower bound corresponding to the $n^{d+1}$ term. The main idea is that if two plausible posterior parameters are typically far apart, then no single estimator value can be close to both. Recall the definition of the prior \Cref{eq:gaussian-prior}. 


\begin{proposition}[Posterior diameter]\label{prop:diameter-main}
Let \(\log(1/\delta)\le q\), and $m\le
  (c\sigma^2 q^d(q+\log(1/\delta)))/(B_d^2\eps^2)$.
Then, conditional on \(\mathcal H_m\), every estimator
\(\smash[t]{\widehat f} \) has
posterior error $\prob(
    \Vert f_{\mu G}-\smash[t]{\widehat f} \Vert_{L_\infty}>\eps
    \mid \mathcal H_m
  )
  \ge 3/8$
under the Gaussian posterior, for \(C_0>0\) large enough.
\end{proposition}

\begin{proof}[Proof sketch]
We condition on the realized transcript and draw two independent samples from the
posterior. By Lemma~4, the posterior covariance still has large trace, and
Lemma~5 converts this residual covariance into a large injective-norm separation
between the two posterior draws with constant probability. Under the chosen
scaling of the prior, this separation transfers to an $L_\infty$ separation
larger than $2\varepsilon$ between the corresponding Walsh polynomials. In this
event, no fixed estimator can be within $\varepsilon$ of both posterior draws.
By symmetry of the draws, the posterior error of any estimator is therefore
bounded below by a positive constant. See \Cref{app:diameter-main} for the proof. 
\end{proof}


\vspace{-0.2cm}

We now handle the $\log(1/\delta)$ term from the lower bound, using the same instance. 
When $\log(1/\delta)>q$, the diameter argument is no longer sharp after the scaling in
\Cref{eq:gaussian-prior}; it loses a factor $\sqrt{q/\log(1/\delta)}$.  The correct
high-confidence term comes from a simpler observation: large posterior trace
implies that at least one cube point still has posterior variance of order
$\mu^2q^d$.  A one-dimensional Gaussian tail then gives the desired
$\delta$-level lower bound.

\begin{proposition}[One-point posterior error]\label{prop:one-point-main}

Let $\log(1/\delta)>q$ and
$m\le c\sigma^2q^d(q+\log(1/\delta))/(B_d^2\eps^2)$. Then, conditional on
$\mathcal H_m$, every estimator ${\widehat f}$ has
posterior error~$\prob\big(
    \Vert f_{\mu G}-{\widehat f}\Vert_{L_\infty}>\eps
    \mid \mathcal H_m
  \big)
  \ge 2\delta$, for $C_0>0$ large enough. 
  
\end{proposition}

The core of the argument relies on applying the probabilistic method using the condition on the trace of $\Sigma_m$ to find a point $x^\star \in \cuben$ which has high enough variance, and then conclude using a Gaussian tail inequality. See \Cref{app:one-point-main} for the full proof.

We now assemble the two regimes.  This final step is a Bayes-to-minimax
reduction: if the average error under a prior supported on the class is large,
then the worst-case error over the class is large.  The only small technical
point is that the Gaussian prior is first conditioned on the event
$\norm{f_{G}}_{L_\infty}\le1$. We can put everything together to prove the lower bound of \Cref{thm:low-degree-main}; see \Cref{app:low-degree-main} for details.

\vspace{-0.2cm}

\section{Sparse Walsh polynomials}\label{sec:sparse}

\vspace{-0.2cm}

We now turn to sparse Walsh polynomials.  In contrast with the low-degree case, the ambient coefficient space is the full Walsh basis of size $2^n$, and the learner must identify a small unknown support.  The main obstruction is therefore combinatorial rather than dimensional: the noisy problem contains a support-search component that is absent from the low-degree least-squares upper bound, where the set of Fourier coefficients that must be zero is always fixed.  This support-search component is also where the gap with the noiseless sparse recovery benchmark appears.  In the exact-query model, sparse
Walsh recovery has complexity $\smash{\widetilde\Theta(sn)}$, while in the noisy $L_\infty$ problem the stochastic support-search term is of order $\sigma^2ns^2/\eps^2$.

\begin{theorem}[Noisy sparse sample complexity]\label{thm:sparse-noisy}
There exist universal constants $c,C,\eps_0>0$ such that, for every $1\le s\le n$, $0<\eps\le\eps_0$, $0<\delta\le 1/4$, and centered $\sigma$-subgaussian noise, we have
\begin{equation}
c \sigma^2 \bigl(ns^2+s\log(1/\delta)\bigr) / \eps^2 \leq m^*(\cS_{n,s},\eps,\delta)   \leq C (\sigma^2 + 1) \bigl(ns^2+s\log(1/\delta)\bigr) / \eps^2.
\end{equation}
Thus, if $\sigma$ is bounded below by a positive universal constant, and $s \leq n$, the above inequalities imply that $m^*(\cS_{n,s},\eps,\delta) = \Theta( (\sigma/\eps)^2 (ns^2+s\log(1/\delta)))$.
\end{theorem}

The proof can be adapted for any $s \leq C n$ by adapting the constants. The upper bound is obtained via a random uniform design and proceeds through empirical Walsh coefficients followed by hard thresholding.  The lower bound has two components.  The first is a support-search lower bound of order $\sigma^2ns^2/\eps^2$, which already holds at constant confidence.  The second is a high-confidence lower bound of order $\sigma^2s\log(1/\delta)/\eps^2$, which already appears on a fixed known support.

\subsection{Sparse sample complexity upper bound}


The sparse case is where the design issue becomes most visible. A query at $x$ observes $\langle \theta,\chi(x)\rangle+\xi$, so the problem can be written as linear regression in the coefficient vector. However, the available design vectors are Walsh--Hadamard rows $\{\chi(x):x\in\cuben\}$, rather than generic independent subgaussian vectors in $\R^{2^n}$. Although these rows are isotropic, their coordinates are not jointly independent, since in particular $\chi_S\chi_T=\chi_{S\symmetricdifference T}$. Thus, standard sparse-regression results such as Lasso bounds do not directly yield the exact guarantee needed here. In particular, the sharp high-confidence term $s\log(1/\delta)$ comes from controlling a localized sparse coefficient norm directly, rather than passing through a coefficient $\Vert\cdot\Vert_2$ bound. The latter route would lose a factor $\sqrt{s}$ when converting to a coefficient $\Vert\cdot\Vert_1$ guarantee for $L_\infty$ error, leading to the weaker term $s^2\log(1/\delta)$. We therefore give a direct argument adapted to the Walsh design and the uniform loss.

The naive method for uniform error would be to control $\| \theta -\widehat \theta \|_1$.  But after thresholding to an $s$-sparse candidate, it is unnecessary to control all coordinates uniformly.  The relevant object is the sum of the largest $s$
coefficient errors.  For a coefficient vector $u=(u_S)_{S\subseteq[n]}$, its top-$s$ norm is defined as 
\begin{equation}\label{eq:Rs-def}
\Vert u \Vert_{(s)} 
\coloneqq
\sup_{|\mathcal{A}|\le s}\sum\nolimits_{S\in \mathcal{A}}|u_S|
=
\sup_{|\mathcal{A}|\le s, \;\eta_S\in\{-1,1\}}
\big|\sum\nolimits_{S\in \mathcal{A}}\eta_S u_S\big|.
\end{equation}
The second expression will be useful, as it identifies the top-$s$ norm as a supremum over signed sparse Walsh test functions. We next show that, when the true coefficient vector $\theta$ is $s$-sparse, hard thresholding converts control of the top-$s$ pre-thresholding error into control of the full coefficient $\Vert \cdot \Vert_1$ error, up to a universal constant; see \Cref{app:top-s-thresholding} for the proof. 
\begin{lemma}[Norm equivalence under sparsity]\label{lem:top-s-thresholding}
Let $\theta,\smash{\widetilde\theta}\in\reals^{2^n}$, with $\theta$ $s$-sparse, and let $\smash[t]{\widehat\theta}$ be obtained by keeping the $s$ largest coordinates of $\smash{\widetilde\theta}$ in absolute value and setting all other coordinates to zero. Then $\Vert \theta - \smash[t]{\widehat \theta} \Vert_{1} \leq 4 \Vert \theta - \smash{\widetilde \theta} \Vert_{(s)} $. Consequently, $\Vert f_\theta - f_{\smash[t]{\widehat\theta}} \Vert_{L_\infty} \leq 4 \Vert \theta - {\widetilde \theta} \Vert_{(s)}$.

\end{lemma}



Due to this lemma and~\eqref{eq:Rs-def}, to bound $\Vert f_\theta - f_{\smash[t]{\widehat \theta}} \Vert_{L_\infty}$ it suffices to uniformly bound the quantity $
\big|\sum_{S\in \mathcal{A}}\eta_S (\theta_S-\tilde{\theta}_S)\big|$ for all $|\mathcal{A}|\le s$ and $\eta\in\{-1,1\}^{\mathcal{A}}$. We next show that the size of this test class is sufficiently small to union-bound over it. See proof in \Cref{app:count-gs}.   

\begin{lemma} \label{lem:count-gs}
There exists a universal constant $C>0$ such that for $s\leq n$, and $\cG_s
\coloneqq
\big\{
g(x)=\sum_{S\in \mathcal{A}}\eta_S\chi_S(x):
|\mathcal{A}|\le s,\ \eta_S\in\{-1,1\}
\big\}$, we have $\log|\cG_s| \leq Csn$.

\end{lemma}

Thus, controlling the top-$s$ norm is a union bound over $\exp(O(sn))$ signed tests. The key point of the upper bound proof is that this test-class control gives the sharp confidence term $s\log(1/\delta)$, rather than the weaker $s^2\log(1/\delta)$.

We use a random uniform design $X_t \sim \mathrm{Unif}(\cuben)$ and estimate each Fourier coefficient by the empirical correlation with the corresponding Walsh
function: $\smash{ \widetilde\theta}_S = (1/m)\sum_{t=1}^m Y_t\chi_S(X_t)$. The estimator $\smash[t]{\widehat \theta}$ then keeps the $s$ largest coordinates of $\smash{\widetilde\theta}$ and
returns as a prediction the corresponding Walsh polynomial $f_{\widehat \theta}$. Observe that a similar approach was used by Eskenazis and Ivanisvili~\cite{eskenazis2022lowdegree}. This polynomial yields the sample complexity upper bound. We now give a sketch of proof, and defer the full proof to \Cref{app:sparse-noisy-upperbound}. 

\vspace{-0.2cm}

\begin{proof}[Proof sketch of the upper bound of \Cref{thm:sparse-noisy}]

Given \Cref{lem:top-s-thresholding}, it suffices to prove that
$\Vert \widetilde\theta-\theta \Vert_{(s)}$ is small with high probability.
Fix a signed sparse test $g=\sum_{S\in \mathcal{A}}\eta_S\chi_S\in\cG_s$. By orthogonality
of the Walsh basis, $\sum\nolimits_{S\in \mathcal{A}}\eta_S(\smash{\widetilde\theta}_S-\theta_S)
  =
  \frac1m\sum\nolimits_{t=1}^m Y_t g(X_t)-\expectation[f(X)g(X)]
  =
  R_{\mathrm{des}}(g)+R_{\mathrm{noise}}(g)$,
where
$R_{\mathrm{des}}(g) \coloneqq m^{-1}\sum_{t=1}^m
(f(X_t)g(X_t)-\expectation[f(X)g(X)])$ and
$R_{\mathrm{noise}}(g) \coloneqq m^{-1}\sum_{t=1}^m\xi_tg(X_t)$. We control these two
terms uniformly over $\cG_s$. For every $g\in\cG_s$, we have
$\norm g_{L_2}^2\le s$, $\norm g_{L_\infty}\le s$, and
$\expectation[g(X)^4]\le s^3$. Bernstein's inequality and a union bound over
$\cG_s$ control the design fluctuation, and the same argument controls the
empirical second moments $m^{-1}\sum_{t=1}^m g(X_t)^2$. Conditional on this
event, subgaussian concentration controls the noise fluctuation uniformly over
$\cG_s$. Since $\log|\cG_s|\le Csn$, this gives the sample size in
\Cref{thm:sparse-noisy}.
\end{proof}

\vspace{-0.3cm}

\subsection{Sparse sample complexity lower bound}\label{subsec:sparse-lower}

For the lower bound, the main challenge, as highlighted in \Cref{sec:model}, is that coefficient separation alone does not imply separation in the target loss: the Boolean $L_\infty$ norm tests only against Walsh--Hadamard rows, so it can be smaller by $\sqrt{s}$ compared to the Fourier coefficient $\| \cdot \|_1$-distance. We resolve this by constructing structured sparse subfamilies for which support disagreement is witnessed directly on the Boolean cube.  One block hides the Walsh frequencies,  while the other lets us choose signs so
that discrepancies between two candidate supports add rather than cancel. Consequently, on the resulting subfamilies, the Hamming distance between hidden supports and the \(L_\infty\) distance between functions are equivalent up to universal constants. Thus, any uniformly accurate learner must solve a genuine support-identification problem over the Walsh dictionary. We suppose for this section that $\xi_t \sim \mathcal{N}(0,\sigma^2)$.

Let $1\le k\le n/2$ and split the cube as
$\cube^k\times\cube^{n-k}$, with points written $x=(u,v) \in \cuben$. We now define the hard family of sparse Walsh polynomials. For a tuple
$M=(S_1,\ldots,S_k)$ with each $S_j\subseteq[n-k]$, define for some constant $a>0$, $F_M(u,v)=\sum\nolimits_{j=1}^k u_j\chi_{S_j}(v)$, and $
f_M=aF_M$.

The first block identifies which sparse coordinate is being queried; the second block hides one of~$2^{n-k}$ possible Walsh frequencies for each coordinate. Let \(M=(S_1,\ldots,S_k)\) and \(T=(T_1,\ldots,T_k)\). Define their Hamming
distance by $
  d_{\mathrm H}(M,T)
  \coloneqq
  \bigl|\{j\in[k]: S_j\ne T_j\}\bigr|$. We next show that this sparse family has the desired Hamming-distance separation
property; see \Cref{app:systematic-family-main} for the proof. 

\begin{lemma}[Structured sparse family]\label{lem:systematic-family-main}
For every $M$, the function $f_M$ is $k$-sparse and
$\norm{f_M}_{L_\infty}\le ak$.  Moreover, for all set tuples $M,T$, $\norm{f_M-f_T}_{L_\infty}
\ge
a\,d_{\mathrm H}(M,T)$.
\end{lemma}


We now prove the bound via a Fano lower bound argument, assuming the noise is exactly Gaussian. 

\begin{proposition}[Support-search lower bound]\label{prop:support-search-main}
There are universal constants $c,\eps_0>0$ such that, for $1\le s\le n$,
$0<\eps\le\eps_0$, if $m \le c \sigma^2ns^2/\eps^2$ then every adaptive estimator has a worst-case error of at least $\sup_{f\in\cS_{n,s}}
\prob_f(\Vert f-{\widehat f}\Vert_{L_\infty}>\eps) \ge \frac12$.
\end{proposition}
\vspace{-1em}

\begin{proof}[Proof sketch]
Take $k$ of order $s$ and set $a=4\eps/k$, so that $ak\le1$ for small enough
$\eps$ and the family is contained in $\cS_{n,s}$.  Let $M$ be uniform over all
tuples $(S_1,\ldots,S_k)$.  By \Cref{lem:systematic-family-main}, an estimator
with $L_\infty$ error at most $\eps$ can be converted into an estimator of
$M$ with Hamming error at most a constant fraction of $k$. It remains to show that the observations do not contain enough
information to recover this tuple.  Condition on all coordinates of $M$ but
one. The remaining unknown has alphabet size $2^{n-k}$.  For two alternatives,
the pointwise difference between the corresponding regression functions is at
most $2a$, so an adaptive KL chain rule gives information at most
$O(ma^2/\sigma^2)$ for that coordinate.  Fano's inequality implies that
the expected Hamming error remains a constant fraction of $k$ unless
$m= \Omega(\sigma^2k^2(n-k)/\eps^2)$.  Since $k\asymp s$ and $s\le n$, this gives
the result.  For the full proof, see \Cref{app:support-search-main}.
\end{proof}

\vspace{-0.3cm}

The second component is a high-confidence lower bound on a fixed known support, which gives the term $\sigma^2s\log(1/\delta)/\eps^2$. Here no support search is needed: the obstruction is simply Gaussian estimation in an $s$-dimensional coefficient space equipped with the $\| \cdot \|_1$-geometry induced by the $L_\infty$ loss. Consider the class of linear functions $\{ f_\theta(x)=\sum_{j=1}^s\theta_jx_j: \norm{\theta}_1\le1 \} \subseteq \cS_{n,s}$. On this subclass, the cube realizes every sign pattern on the first $s$ coordinates, and therefore $\norm{f_\theta-f_{\theta'}}_{L_\infty} = \norm{\theta-\theta'}_1$. Thus, the sparse Walsh problem contains a~$\| \cdot \|_1$ Gaussian linear-estimation problem on a known support. Combining these two obstructions gives the claimed sparse lower bound. See \Cref{app:sparse-highconf-main} for a full proof. 

\begin{lemma}\label{prop:sparse-highconf-main}
There are universal constants $c,\eps_0>0$ such that, if $\log(1/\delta)\ge s$, $0<\eps\le\eps_0$, and $m \le c \sigma^2s \log(1/\delta)/\eps^2$, then every adaptive estimator satisfies $\sup_{f\in\cS_{n,s}}
\prob_f(\Vert f - {\widehat f} \Vert_{L_\infty}>\eps) \ge \delta$.
\end{lemma}

\vspace{-0.2cm}

\section{Open directions} \label{sec:openquestions}

\vspace{-0.2cm}


Several questions remain open. In the sparse upper bound, the factor
\(\sigma^2+1\) suggests that our estimator does not fully separate the noiseless
design cost from the stochastic noise contribution. It would be interesting to
know if sparse Fourier-type arguments for Walsh--Hadamard designs can yield a sharper
low-noise guarantee with a purely \(\sigma^2\)-dependent stochastic term. For
low-degree polynomials, the optimal dependence on the degree also remains
unclear, particularly the role of the Boolean Bohnenblust--Hille constant in the
lower bound. Finally, the sparse theorem is most meaningful when \(s\lesssim n\).
Extending the characterization to larger sparsity, such as \(s=\omega(n)\),
requires new ideas. While the upper bound extends without issues, the current lower bound \(ns^2\) scaling cannot persist for $s \geq n$, as can be seen for $s=2^n$. More broadly, it remains open to identify natural polynomial classes
for which adaptive querying can improve the minimax rate.

\section*{Acknowledgements}

Jasper van Doornmalen thanks the support of ANID FONDECYT Postdoctorado Nr.\@~3250132.
Mathieu Molina thanks the support of the European Research Council (ERC) under the European Union's Horizon Europe Program (FACT, grant agreement No.~101170373).
José Verschae thanks the support of ANID FONDECYT Regular Nr.\@~1260640 and the Center for Mathematical Modeling (CMM) Basal fund~FB210005.
Victor Verdugo thanks the support of ANID FONDECYT Regular Nr.\@~1241846.

{
\small
\bibliographystyle{plainnat}
\bibliography{bibliography}
}

\newpage 

\appendix

\section{Tools for upper and lower-bounds}


\subsection{Walsh polynomials}

\begin{lemma}[Moments of signed Walsh sums]\label{lem:signed-moments}
Let $\cA\subseteq 2^{[n]}$ satisfy $|\cA|=k$, and let $g=\sum_{S\in \cA}\eta_S\chi_S \in \cG_s$. 
If $X\sim\Unif(\cuben)$, then
\begin{equation*}
    \norm{g}_{L_\infty}\le k,
    \qquad
    \expectation[g(X)^2]=k,
    \qquad
    \expectation[g(X)^4]\le k^3 .
\end{equation*}
\end{lemma}

\begin{proof}
For every $x\in\cuben$, since each Walsh function takes values in $\cube$,
\begin{equation*}
    |g(x)|
    \le
    \sum_{S\in \cA} |\eta_S\chi_S(x)|
    =
    k .
\end{equation*}
Thus, $\norm{g}_{L_\infty}\le k$.

For the second moment, by orthogonality of the Walsh basis,
\begin{align*}
    \expectation[g(X)^2]
    &=
    \sum_{S,T\in \cA}
    \eta_S\eta_T\expectation[\chi_S(X)\chi_T(X)]  \\
    &=
    \sum_{S,T\in \cA}
    \eta_S\eta_T\indicator\{S=T\}
    =
    \sum_{S\in \cA}\eta_S^2
    =
    k .
\end{align*}

For the fourth moment, expanding gives
\begin{equation*}
    \expectation[g(X)^4]
    =
    \sum_{S_1,S_2,S_3,S_4\in \cA}
    \eta_{S_1}\eta_{S_2}\eta_{S_3}\eta_{S_4}
    \expectation[
        \chi_{S_1}(X)\chi_{S_2}(X)\chi_{S_3}(X)\chi_{S_4}(X)
    ] .
\end{equation*}
The expectation is nonzero only when the product Walsh function is constant,
i.e., when $S_1 \symmetricdifference S_2 \symmetricdifference S_3 \symmetricdifference S_4=\emptyset$.
Equivalently, $S_4=S_1\symmetricdifference S_2\symmetricdifference S_3$, so $S_4$ is determined by
$S_1,S_2,S_3$. Hence, at most $k^3$ quadruples contribute to the sum, and each
contributing term has absolute value at most $1$. Therefore, $    \expectation[g(X)^4]\le k^3 $.
\end{proof}

\subsection{Some properties of the Gaussian distribution}

\begin{lemma}[Gaussian self-conjugacy]
\label{lem:posterior-cov}
Let $\theta\sim\mathcal{N}(0,\Sigma_0)$ in $\reals^b$, and suppose that,
conditionally on adaptively chosen query vectors $a_t\in\reals^b$,  $Y_t=\langle a_t,\theta\rangle+\xi_t$, $\xi_t\sim\mathcal{N}(0,\sigma^2)$ independently across $t$. Conditional on the realized query vectors and
observations, the posterior law of $\theta$ is Gaussian with covariance
\begin{equation*}
    \Sigma_m
    =
    \left(
        \Sigma_0^{-1}
        +
        \frac{1}{\sigma^2}\sum_{t=1}^m a_ta_t^\top
    \right)^{-1}.
\end{equation*}
\end{lemma}

\begin{proof}
Fix realized query vectors $a_1,\ldots,a_m\in\reals^b$ and observations
$y_1,\ldots,y_m$. The prior density of $\theta$ is proportional to
\begin{equation*}
    \exp\left(
        -\frac{1}{2}\theta^\top\Sigma_0^{-1}\theta
    \right).
\end{equation*}
Conditionally on $\theta$, the observation density is proportional to
\begin{equation*}
    \exp\left(
        -\frac{1}{2\sigma^2}
        \sum_{t=1}^m
        (y_t-\langle a_t,\theta\rangle)^2
    \right).
\end{equation*}
Thus, by Bayes' rule, the posterior density is proportional to
\begin{align*}
    \exp\bigg(
        -\frac{1}{2}\theta^\top\Sigma_0^{-1}\theta
        -
        \frac{1}{2\sigma^2}
        \sum_{t=1}^m
        (y_t-\langle a_t,\theta\rangle)^2
    \bigg).
\end{align*}
Expanding the quadratic term gives
\begin{align*}
    \sum_{t=1}^m (y_t-\langle a_t,\theta\rangle)^2
    &=
    \sum_{t=1}^m y_t^2
    -
    2\sum_{t=1}^m y_t a_t^\top\theta
    +
    \theta^\top
    \left(\sum_{t=1}^m a_ta_t^\top\right)
    \theta .
\end{align*}
The terms independent of $\theta$ are absorbed into the normalizing constant.
Therefore, the posterior density is proportional in $\theta$ to
\begin{align*}
    \exp\bigg(
        &-\frac{1}{2}
        \theta^\top
        \left(
            \Sigma_0^{-1}
            +
            \frac{1}{\sigma^2}\sum_{t=1}^m a_ta_t^\top
        \right)
        \theta
        +
        \left(
            \frac{1}{\sigma^2}\sum_{t=1}^m y_ta_t
        \right)^\top
        \theta
    \bigg).
\end{align*}
This is a Gaussian density. Completing the square shows that the posterior covariance matrix is as claimed:
\begin{equation*}
    \Sigma_0^{-1}
    +
    \frac{1}{\sigma^2}\sum_{t=1}^m a_ta_t^\top.\qedhere
\end{equation*}
\end{proof}

\begin{lemma}[Gaussian logarithmic tail]\label{lem:gauss-log-tail}
Fix $K>0$ and $L_0\ge 1$ such that $Ke^{-L_0}<1$. Then there exists
$a_0=a_0(K,L_0)>0$ such that, for every $L\ge L_0$ and every
$Z\sim\mathcal{N}(0,1)$,
\begin{equation*}
    \prob(|Z|\ge a_0\sqrt L)\ge Ke^{-L}.
\end{equation*}
\end{lemma}

\begin{proof}
Let $p(u)\coloneqq\prob(|Z|\ge u)$. Choose any $b\in(0,\sqrt 2)$. By Mill's ratio lower bound (see e.g. Proposition 2.1.2 of \cite{Vershynin_2018}), we have  $p(u)\ge\frac{1}{\sqrt{2\pi}}\frac{u}{1+u^2}e^{-u^2/2}$ for $u\ge 0$. So there exists a universal constant $c>0$
such that $p(u)\ge c u^{-1}e^{-u^2/2}$ when $ u\ge 1 $. Hence,
\begin{equation*}
    e^L p(b\sqrt L) \geq \frac{c}{b\sqrt L} e^{(1-b^2/2)L}\xrightarrow[L \to \infty ]{} \infty,
\end{equation*}
because $b^2/2<1$. Therefore, there exists $L_1\ge L_0$ such that
\begin{equation*}
    p(b\sqrt L)\ge Ke^{-L}
    \qquad\text{for all } L\ge L_1 .
\end{equation*}

It remains to consider $L\in[L_0,L_1]$. Since $Ke^{-L_0}<1=p(0)$ and $p$ is
continuous at $0$, and because $K e^{-L_0}<1$, we can choose $a_0\in(0,b]$ small enough so that
\begin{equation*}
    p(a_0\sqrt{L_1})\ge Ke^{-L_0}.
\end{equation*}
Then, for every $L\in[L_0,L_1]$, monotonicity of $p$ gives
\begin{equation*}
    p(a_0\sqrt L)
    \ge
    p(a_0\sqrt{L_1})
    \ge
    Ke^{-L_0}
    \ge
    Ke^{-L}.
\end{equation*}
For $L\ge L_1$, since $a_0\le b$ and $p$ is decreasing,
\begin{equation*}
    p(a_0\sqrt L)
    \ge
    p(b\sqrt L)
    \ge
    Ke^{-L}.
\end{equation*}
This proves the claim.
\end{proof}

\subsection{Information theoretic lower bounds}

The following Lemma is the standard divergence decomposition for adaptive experiments, see,
for example, the arguments of Lemma 15.1 of \citet{lattimore2020bandit}. 

\begin{lemma}[KL divergence under adaptive queries]\label{lem:adaptive-kl}
Let $f$ and $g$ be two candidate regression functions, and suppose that the
noise is Gaussian with law $\mathcal{N}(0,\sigma^2)$. Under any adaptive design
rule,
\begin{equation*}
    \operatorname{KL}(P_f^{(m)}\|P_g^{(m)})
    =
    \frac{1}{2\sigma^2}
    \sum_{t=1}^m
    \expectation_f\big[(f(X_t)-g(X_t))^2\big]
    \le
    \frac{m}{2\sigma^2}\norm{f-g}_{L_\infty}^2 .
\end{equation*}
\end{lemma}

\begin{proof}
Recall that $\mathcal H_t=(X_1,Y_1,\ldots,X_{t-1},Y_{t-1})$ denotes the
observations available before choosing query $X_t$. The adaptive design rule
specifies, for each $t$, a conditional distribution for $X_t$ given
$\mathcal H_t$. This conditional distribution is the same under $P_f$ and
$P_g$ and only the conditional law of $Y_t$ given $(\mathcal H_t,X_t)$ changes.

By the chain rule for relative entropy,
\begin{align*}
    \operatorname{KL}(P_f^{(m)}\|P_g^{(m)})
    =
    \sum_{t=1}^m
    \expectation_f\bigg[
        \operatorname{KL}\big(
            P_f(X_t,Y_t\mid \mathcal H_t)
            \,\big\|\,
            P_g(X_t,Y_t\mid \mathcal H_t)
        \big)
    \bigg].
\end{align*}
For each realization of $\mathcal H_t$, the conditional law of $X_t$ is the
same under $f$ and $g$. Therefore, the contribution of $X_t$ to the conditional
KL is zero, and the previous display reduces to
\begin{align*}
    \operatorname{KL}(P_f^{(m)}\|P_g^{(m)})
    =
    \sum_{t=1}^m
    \expectation_f\bigg[
        \operatorname{KL}\big(
            P_f(Y_t\mid \mathcal H_t,X_t)
            \,\big\|\,
            P_g(Y_t\mid \mathcal H_t,X_t)
        \big)
    \bigg].
\end{align*}
Given $\mathcal H_t$ and $X_t$, the two conditional laws are
$\mathcal{N}(f(X_t),\sigma^2)$ and $\mathcal{N}(g(X_t),\sigma^2)$. Since
\begin{equation*}
    \operatorname{KL}\big(
        \mathcal{N}(a,\sigma^2)
        \,\big\|\,
        \mathcal{N}(b,\sigma^2)
    \big)
    =
    \frac{(a-b)^2}{2\sigma^2},
\end{equation*}
we obtain
\begin{equation*}
    \operatorname{KL}(P_f^{(m)}\|P_g^{(m)})
    =
    \frac{1}{2\sigma^2}
    \sum_{t=1}^m
    \expectation_f\big[(f(X_t)-g(X_t))^2\big].
\end{equation*}
Finally, $(f(X_t)-g(X_t))^2\le\norm{f-g}_{L_\infty}^2$ almost surely for every
$t$, which gives
\begin{equation*}
    \operatorname{KL}(P_f^{(m)}\|P_g^{(m)})
    \le
    \frac{m}{2\sigma^2}\norm{f-g}_{L_\infty}^2 . \qedhere
\end{equation*}
\end{proof}

We recall Fano's inequality, see Theorem 4.10 of \citet{rigollet2023hds}.

\begin{lemma}[Fano's inequality]\label{lem:fano}
Let $F$ be uniformly distributed on a finite set $\cF$ with
$|\cF|=M\ge 2$, and let $Z$ be an observation whose law under
$F=f$ is $P_f$. Then every estimator $\smash[t]{\widehat \Theta}(Z)$ satisfies 
\begin{equation*}
    \prob(\smash[t]{\widehat f}\neq f)
    \ge
    1-\frac{\frac{1}{M^2}
    \sum_{f,f'\in\cF}
    \operatorname{KL}(P_f\|P_{f'})+\log 2}{\log M}.
\end{equation*}
\end{lemma}

We state the following Lemma which allows us to condition on the prior:

\begin{lemma}[Conditioning on boundedness]\label{lem:prior-conditioning}
Let $\pi$ be a prior on functions and let $B=\{\norm{f}_{L_\infty}\le 1\}$ satisfy $\pi(B)\ge 1-\eta$. Let $E=\{\norm{\smash[t]{\widehat f}-f}_{L_\infty}>\eps\}$. If $\prob_{f\sim\pi}(E)\ge \rho$, then, under the conditional prior $\pi(\cdot\mid B)$, $\prob_{f\sim\pi(\cdot\mid B)}(E)\ge \rho-\eta$.
Consequently, if $\pi(\cdot\mid B)$ is supported on $\mathcal F$, then
\begin{equation*}
    \sup_{f\in\mathcal F}
    \prob_f(\norm{\smash[t]{\widehat f}-f}_{L_\infty}>\eps)
    \ge
    \rho-\eta .
\end{equation*}
\end{lemma}

\begin{proof}
By the law of total probability,
\begin{align*}
    \prob_{f\sim\pi}(E)
    &=
    \prob_{f\sim\pi}(E\mid B)\pi(B)
    +
    \prob_{f\sim\pi}(E\mid B^c)\pi(B^c)  \\
    &\le
    \prob_{f\sim\pi}(E\mid B)+\pi(B^c)  \\
    &\le
    \prob_{f\sim\pi(\cdot\mid B)}(E)+\eta .
\end{align*}
Thus, $\prob_{f\sim\pi(\cdot\mid B)}(E)\ge \rho-\eta$.

If $\pi(\cdot\mid B)$ is supported on $\mathcal F$, then the average risk under
$\pi(\cdot\mid B)$ is at most the worst-case risk over $\mathcal F$. Therefore,
\begin{equation*}
    \sup_{f\in\mathcal F}
    \prob_f(\norm{\smash[t]{\widehat f}-f}_{L_\infty}>\eps)
    \ge
    \prob_{f\sim\pi(\cdot\mid B)}(E)
    \ge
    \rho-\eta . \qedhere
\end{equation*}
\end{proof}

\subsection{Concentration Inequalities}

We recall the main concentration inequalities that will be used later on. All of these concentration inequalities can be found in \citet{Vershynin_2018}. 

\begin{lemma}[Bernstein, Theorem 2.9.1 of \cite{Vershynin_2018}]
\label{lem:bernstein}
If $Z_1,\dots,Z_m$ are independent, centered, $|Z_t|\le B$ almost surely, and $\sum_t \expectation Z_t^2\le v$, then for every $r>0$,
\begin{equation*}
\prob\left(\left|\sum_{t=1}^m Z_t\right|>r\right)
\le 2\exp\left(-c\min\left\{\frac{r^2}{v},\frac{r}{B}\right\}\right)
\end{equation*}
for a universal constant $c>0$.
\end{lemma}

\begin{lemma}[Hoeffding for subgaussian weighted sums, Theorem 2.7.3 of \cite{Vershynin_2018}]
\label{lem:subg-weighted}
If $\xi_1,\dots,\xi_m$ are independent centered $\sigma$-subgaussian random variables, then for every deterministic $a\in\R^m$ and every $r>0$,
\begin{equation*}
\prob\left(\left|\sum_{t=1}^m a_t\xi_t\right|>r\right)
\le 2\exp\left(-\frac{r^2}{2\sigma^2\sum_t a_t^2}\right).
\end{equation*}
\end{lemma}

\begin{lemma}[Matrix Bernstein, Theorem 5.4.1. of \cite{Vershynin_2018}] \label{lem:matrix_bernstein}
Let $Z_1,\ldots,Z_m$ be independent, symmetric, mean-zero $D\times D$
random matrices. We denote by $\norm{Z_t}_{\mathrm{op}}$ the operator norm of $Z_t$. Suppose that $\norm{Z_t}_{\mathrm{op}}\le K$ almost surely for all $t$. Define the variance proxy
\begin{equation*}
    \sigma^2
    =
    \norm{
        \sum_{t=1}^m \expectation[Z_t^2]
    }_{\mathrm{op}} .
\end{equation*}
Then, for every $u\ge 0$,
\begin{equation*}
    \prob\left(
        \norm{\sum_{t=1}^m Z_t}_{\mathrm{op}}\ge u
    \right)
    \le
    2D
    \exp\left(
        -c\min\left\{
            \frac{u^2}{\sigma^2},
            \frac{u}{K}
        \right\}
    \right),
\end{equation*}
where $c>0$ is a universal constant.
\end{lemma}

\section{Proofs for the noiseless recovery case}

\subsection{Noiseless low degree} \label{app:noiseless-recovery-low}

\begin{proposition}
Let $\delta,\eps \in (0,1)$. For deterministic adaptive algorithms, the sample complexity of noiseless approximate high probability $L_\infty$ recovery is exactly $D$. For $0<2\delta + \eps< 1/30$, the sample complexity for adaptive randomized algorithms is $D/2 \leq m^{(0)}(\cP_{n,d},\eps,\delta)\leq D$.
\end{proposition}

\begin{proof}
For the upper bound, by Theorem $10$ of \cite{eskenazis2024metricentropy} the exact recovery can be done in $D$ queries. \medskip

First, we prove a lower bound for deterministic algorithms, and then leverage the recent work of \cite{Krieg2025} to lift this lower bound to randomized algorithms.

\paragraph{Deterministic lower bound.} We prove the lower bound by indistinguishability. Fix a deterministic exact-query learner using
$m<D$ queries, and run it on the zero function. This determines queried points
$x_1,\ldots,x_m\in\cuben$, all with observed value zero. Since these queries impose only $m$
homogeneous linear constraints on the $D$-dimensional space of degree-at-most-$d$ Walsh
polynomials, there exists a nonzero polynomial $h$ of degree at most $d$ such that
$h(x_t)=0$ for all $t\in[m]$. Rescaling, assume $\norm{h}_{L_\infty}=1$.

Let $f_+=h$ and $f_-=-h$. Then $f_+,f_-\in\cP_{n,d}$, and both functions give the same
answers to all queried points:
\begin{equation*}
    f_+(x_t)=f_-(x_t)=0
    \qquad \text{for all } t\in[m].
\end{equation*}
Thus, the learner follows the same transcript and outputs the same hypothesis $\widehat f$ on
both targets. But
\begin{equation*}
    \norm{f_+-f_-}_{L_\infty}=2.
\end{equation*}
Hence, no single $\widehat f$ can be within $L_\infty$-error $\eps<1$ of both $f_+$ and $f_-$, since otherwise
\begin{equation*}
    2
    =
    \norm{f_+-f_-}_{L_\infty}
    \le
    \norm{\widehat f-f_+}_{L_\infty}
    +
    \norm{\widehat f-f_-}_{L_\infty}
    \le
    2\eps
    <
    2.
\end{equation*}
Therefore, any deterministic learner achieving error $\eps<1$ on all functions in $\cP_{n,d}$
must use at least $D$ exact queries.

\paragraph{Randomized lower bound via Bernstein numbers.}
The preceding lower bound proves that deterministic exact-query learners require $D$ queries. It also applies to randomized learners with success probability one, by fixing a successful realization of the internal randomness. For randomized learners with failure probability $\delta>0$, however, the same nullspace argument does not immediately imply a $D$-query lower bound for the bounded class $\cP_{n,d}$: the hard pair $f^+,f^-$ may depend on the realized transcript, whereas the guarantee only needs to hold with high probability over the learner's randomness.
 
For Banach spaces $(X,\lVert\cdot\rVert_X)$ and $(Y,\lVert\cdot\rVert_Y)$, a continuous linear operator \(S\in\mathcal L(X,Y)\), and
a convex input class \(F\subseteq X\)define the Bernstein number as follows 
\begin{align*}
\label{eq:bernstein-definition}
    b_k(S,F) 
    \coloneqq \!\!\!\!\!\!\!\!\!\!
    \sup_{\substack{V\subseteq X:\\ \dim(V)=k+1,\\ S\text{ injective on }V}}\!\!\!\!\!\!\!\!\!\!
    \left\{
        r>0:
        g+B\subseteq F
        \text{ for some }g\in F
        \text{ and some ball }B\text{ of radius }r
        \text{ in }(V,\norm{\cdot}_S)
    \right\},
\end{align*}
where $\norm{h}_S\coloneqq\norm{Sh}_Y$. Let \(e_m^{\mathrm{ran}}(S,F)\) denotes the minimal worst-case expected error in $L_{\infty}$
over adaptive randomized algorithms using \(m\) arbitrary linear measurements. Here, linear measurements correspond to point evaluations.

Let $\cF_{n,d}$ be $\cP_{n,d}$ but without the boundedness condition, so that $\cP_{n,d} \subset \cF_{n,d}$ is a convex subset. In our setting, we take $X=\cF_{n,d}$ equipped with the $L_\infty$ norm $\norm{f}_{X}\coloneqq \norm{f}_{L_\infty}$, $S=\id:\cF_{n,d}\to L_\infty(\cuben)$, and $Y=L_\infty(\cuben)$. With this choice, $\mathrm{dim}(\cF_{n,d})=D$ and the class $\cP_{n,d}$ is exactly the unit ball of $X$. Moreover, the norm induced by $S$ is the same norm, $\Vert h \Vert_S=\Vert Sh\Vert_Y=\Vert h \Vert_{L_\infty}=\Vert h \Vert_{X}$. 

We now compute the Bernstein numbers in this setting.  For every
\(0\le k\le D-1\),
\begin{equation}
\label{eq:bernstein-low-degree-value}
    b_k(\id,\cP_{n,d})=1.
\end{equation}
Indeed, let \(V\subseteq X\) be any subspace of dimension \(k+1\).  Since
\(k+1\le D\), such subspaces exist, and \(\id\) is injective on \(V\).  The unit
ball of \(V\), measured in the norm \(\norm{\cdot}_S\), is
\begin{equation*}
    B_V
    =
    \{h\in V:\norm{h}_S\le1\}
    =
    \{h\in V:\norm{h}_{L_\infty}\le1\}.
\end{equation*}
Since \(\cP_{n,d}\) is the unit ball of \(\cF_{n,d}\), we have $0+B_V\subseteq \cP_{n,d}$. So all balls of radius $r=1$ are properly included, and therefore \(b_k(\id,\cP_{n,d})\ge1\).
Conversely, no ball of radius strictly larger than \(1\) can be contained in
\(\cP_{n,d}\).  Indeed, if \(g+B\subseteq \cP_{n,d}\), where \(B\) is a ball of
radius \(r\) in \((V,\norm{\cdot}_S)\), then for any nonzero \(h\in V\) with
\(\norm{h}_S=1\), both \(g+rh\) and \(g-rh\) belong to \(\cP_{n,d}\). Hence,
\begin{equation*}
    2r
    =
    \norm{(g+rh)-(g-rh)}_{L_\infty}
    \le
    \norm{g+rh}_{L_\infty}
    +
    \norm{g-rh}_{L_\infty}
    \le
    2,
\end{equation*}
so \(r\le1\).  This proves \eqref{eq:bernstein-low-degree-value}.

Now suppose that a randomized adaptive noiseless learner uses \(m\le D/2\)
queries.  Then \(2m-1\le D-1\), and the Bernstein lower bound of
\citet[Section~4, Eq.~(6)]{Krieg2025} gives
\begin{equation}
\label{eq:bernstein-expected-error-low-degree}
    e_m^{\mathrm{ran}}(\id,\cP_{n,d})
    \ge
    \frac1{30} b_{2m-1}(\id,\cP_{n,d})
    =
    \frac1{30}.
\end{equation}
This lower bound is for the stronger model where the learner may use arbitrary
linear measurements.  It therefore also applies to point-query learners, because
for every \(x\in\cuben\), the map \(f\mapsto f(x)\) is a continuous linear
functional on \(\cF_{n,d}\): $|f(x)|
    \le
    \norm{f}_{L_\infty}
    =
    \norm{f}_{X}$.

It remains to convert the expected-error lower bound into our high-probability
formulation.  Suppose, for contradiction, that there exists an \(m\)-query
randomized noiseless learner with \(m\le D/2\) such that, for every
\(f\in\cP_{n,d}\),
\begin{equation*}
    \prob_f\left(
        \norm{f-\smash[t]{\widehat f}}_{L_\infty}\le \eps
    \right)
    \ge
    1-\delta .
\end{equation*}
Clip the output pointwise to \([-1,1]\), and denote the clipped estimator by
\(\smash[t]{\widetilde f}\).  Since every \(f\in\cP_{n,d}\) takes values in
\([-1,1]\), clipping cannot increase the pointwise error. Thus,
\begin{equation*}
    \prob_f\left(
        \norm{f-\smash[t]{\widetilde f}}_{L_\infty}\le \eps
    \right)
    \ge
    1-\delta .
\end{equation*}
Moreover, $\norm{f-\smash[t]{\widetilde f}}_{L_\infty}\le2$ 
deterministically.  Therefore, bounding the error on the good and bad event, for every \(f\in\cP_{n,d}\),
\begin{equation*}
    \expectation_f
    \norm{f-\smash[t]{\widetilde f}}_{L_\infty}
    \le
    (1-\delta)\eps+2\delta \leq \eps +2\delta  .
\end{equation*}
If $\eps+2\delta<\frac1{30}$,
this contradicts \eqref{eq:bernstein-expected-error-low-degree}.  Consequently, $m^{(0)}(\cP_{n,d},\eps,\delta)
    >\frac{D}{2}$ whenever the condition on $\eps$ and $\delta$ holds. \qedhere

\end{proof}

\subsection{Noiseless sparse recovery} \label{app:noiseless-recovery-sparse}

\begin{proposition}
There exist universal constants $C,c,\eps_0,\delta_0>0$ such that, for every
$1\le s\le n$, $0<\eps\le\eps_0<1 $, and $0<\delta\le\delta_0$,
\begin{equation*}
c\frac{sn}{\log(2s)} \leq     m^{(0)}(\cS_{n,s},\eps,\delta)
    \leq C sn \log^2(2s).
\end{equation*}
Therefore, $m^{(0)}(\cS_{n,s},\eps,\delta) = \operatorname{ \widetilde  \Theta}(sn)$ for $1\le s\le n$, $0<\eps\le\eps_0$, and $0<\delta\le\delta_0$.
\end{proposition}

\begin{proof}
We first prove the upper bound. Let $N=2^n$. A noiseless query at
$x\in\cuben$ is a linear measurement of the Fourier coefficient vector
$\theta\in\reals^N$, with measurement row $(\chi_S(x))_{S\subseteq[n]}$.
Thus, using $m$ noiseless queries amounts to selecting $m$ rows of the
$N\times N$ Walsh--Hadamard matrix.

By the restricted-isometry theorem for subsampled Fourier matrices of
\citet[Theorem 2.1]{Haviv2016}, applied to the Fourier matrix of
$(\integers_2)^n$ at sparsity level $2s$, there exists a universal constant
$C>0$ and a choice of
\begin{equation*}
    m\le Cs n\log^2(2s)
\end{equation*}
rows such that the normalized sampling matrix is injective on $2s$-sparse
vectors. Indeed, if the image of the sampled matrix $A$ at a point $z$ is $\Vert Az \Vert=0$, the RIP condition yields $(1-\eta) \Vert z \Vert \leq \Vert A z \Vert=0$, implying that $z=0$ whenever $\eta<1$, where a smaller $\eta$ makes $C$ larger. Taking $\eta=1/2$ is sufficient. Consequently, if two $s$-sparse coefficient vectors $\theta,\theta'$
give the same observations, then $\theta-\theta'$ is a $2s$-sparse vector in
the kernel, and hence $\theta=\theta'$. The learner may therefore output the
unique $s$-sparse coefficient vector consistent with the observations. This
recovers $f$ exactly, so
\begin{equation*}
    m^{(0)}(\cS_{n,s},\eps,\delta)
    \le
    Cs n\log^2(2s).
\end{equation*}

We now prove the lower bound. The case $s=1$ follows by considering the family
$\{\chi_S:S\subseteq[n]\}$: it has size $2^n$, is $2$-separated in
$L_\infty$, and each exact query has only two possible answers. Thus,
$m\ge c n$ for $\eps<1$, and $\delta\le\delta_0$. We henceforth assume
$s\ge2$.

Set $k\coloneqq \left\lfloor \frac{s}{2}\right\rfloor $, and $a\coloneqq \frac1k$.
Since $s\le n$, we have $k\le n/2$. Apply
\Cref{lem:systematic-family-main} with this value of $k$. For every parameter
$M=(S_1,\ldots,S_k)$, the function $f_M$ is $k$-sparse and satisfies
\begin{equation*}
    \norm{f_M}_{L_\infty}\le ak=1.
\end{equation*}
Thus, $f_M\in\cS_{n,s}$.

We restrict the parameter set to a large Hamming code. The full parameter set is
$[q]^k$ with alphabet size $q=2^{n-k}$. By a standard greedy packing argument,
there exists a subset $\mathcal C\subseteq [q]^k$ such that, for all distinct
$M,M'\in\mathcal C$,
\begin{equation*}
    d_{\mathrm H}(M,M')\ge \frac{k}{4}, \qquad  \mbox{and}\quad  \log|\mathcal C|\ge c k(n-k),
\end{equation*}
for a universal constant $c>0$.

Let \(q=2^{n-k}\), so the set of all tuples \(M=(S_1,\dots,S_k)\) is identified with the \(q\)-ary Hamming cube \([q]^k\). We construct \(\mathcal C\) greedily: repeatedly choose an unused point \(M\in[q]^k\), add it to \(\mathcal C\), and delete all points \(T\) with \(d_H(M,T)<k/4\). By construction, any two distinct \(M,T\in\mathcal C\) satisfy \(d_H(M,T)\ge k/4\). It remains to lower bound \(|\mathcal C|\). Each chosen codeword deletes at most one Hamming ball of radius \(k/4\), whose size is bounded by
\[
\sum_{\ell\le k/4}{k\choose \ell}(q-1)^\ell
\le
\sum_{\ell\le k/4}{k\choose \ell}q^\ell
\le
2^k q^{k/4}.
\]
Since the whole space has size \(q^k\), the greedy procedure must select at least
\[
|\mathcal C|
\ge
\frac{q^k}{2^kq^{k/4}}
=
2^{-k}q^{3k/4}
\]
codewords. Therefore, using \(q=2^{n-k}\),
\[
\log|\mathcal C|
\ge
\frac{3k}{4}\log q-k\log 2
=
\left(\frac{3}{4}k(n-k)-k\right)\log 2.
\]
Since \(k\le n/2\), this is at least \(c\,k(n-k)\) for a universal constant \(c>0\).

For $M\ne M'$ in $\mathcal C$, \Cref{lem:systematic-family-main} gives
\begin{equation*}
    \norm{f_M-f_{M'}}_{L_\infty}
    \ge
    a\,d_{\mathrm H}(M,M')
    \ge
    \frac14 .
\end{equation*}
Hence, if $\eps_0<1/8$, any hypothesis with $L_\infty$-error at most $\eps\le
\eps_0$ identifies the element of $\mathcal C$ uniquely.

For any query $(u,v)\in\{-1,1\}^k\times\{-1,1\}^{n-k}$, $f_M(u,v)
    =
    \frac1k\sum_{j=1}^k \pm 1$,
so each exact answer belongs to a set of size at most $2k+1$. Thus, a
deterministic adaptive decision tree of depth $m$ has at most $(2k+1)^m$
possible transcripts, and each transcript can be correct for at most one
parameter in $\mathcal C$.

We now extend the counting argument to randomized learners. Let \(\mathcal C\subseteq[q]^k\) be the code constructed above, and let \(M\sim \mathrm{Unif}(\mathcal C)\). Consider any randomized \(m\)-query learner \(A\), and denote its internal randomness by \(R\). After conditioning on \(R=r\), the learner becomes a deterministic adaptive decision tree \(A_r\) of depth \(m\). Since every noiseless query value \(f_M(u,v)\) belongs to a set of cardinality at most \(2k+1\), this decision tree has at most \((2k+1)^m\) possible leaves.

Fix \(r\). For a leaf \(\ell\), the deterministic learner \(A_r\) outputs some function \(\widehat f_\ell\). This leaf can be successful for at most one \(M\in\mathcal C\). Indeed, if it were successful for two distinct \(M,T\in\mathcal C\), then
\[
\|f_M-\widehat f_\ell\|_{L_\infty}\le \varepsilon
\qquad\text{and}\qquad
\|f_T-\widehat f_\ell\|_{L_\infty}\le \varepsilon,
\]
and therefore
\[
\|f_M-f_T\|_{L_\infty}\le 2\varepsilon,
\]
contradicting the code separation \(\|f_M-f_T\|_{L_\infty}\ge 1/4\) whenever \(\varepsilon<1/8\). Hence, for each fixed \(r\), the deterministic learner can be correct on at most one codeword per leaf, and so on at most \((2k+1)^m\) codewords in total. Consequently,
\[
\mathbb P_{M\sim \mathrm{Unif}(\mathcal C)}
\bigl(\|A_r(f_M)-f_M\|_{L_\infty}\le \varepsilon\bigr)
\le
\frac{(2k+1)^m}{|\mathcal C|}.
\]

Averaging over \(R\), we get $\mathbb P_{M,R}
\bigl(\|A_R(f_M)-f_M\|_{L_\infty}\le \varepsilon\bigr)
\le
\frac{(2k+1)^m}{|\mathcal C|}$.
On the other hand, if \(A\) succeeded uniformly with failure probability at most \(\delta\), then for every \(M\in\mathcal C\), $\mathbb P_{M,R}
\bigl(\|A_R(f_M)-f_M\|_{L_\infty}\le \varepsilon\bigr)
\ge 1-\delta$.
Therefore, it is necessary that
\[
(2k+1)^m\ge (1-\delta)|\mathcal C|.
\]
Taking logarithms gives
\[
m\ge
\frac{\log|\mathcal C|+\log(1-\delta)}{\log(2k+1)}.
\]
Using \(\log|\mathcal C|\ge c k(n-k)\), \(k=\lfloor s/2\rfloor\), and \(s\le n\), this yields $m\ge c'\frac{sn}{\log(2s)}$
for another universal constant \(c'>0\), provided \(\delta\) is bounded away from \(1\). \qedhere

\end{proof}

\section{Omitted proofs of \Cref{sec:low-degree}}

\subsection{Proof of \Cref{lem:bounded-prior}} \label{app:bounded-prior}

\begin{proof}
Fix $x\in\cuben$. Since $G\sim\mathcal{N}(0,I_b)$ and
$\norm{\chi_{\cB}(x)}_2^2=q^d$, we have $\langle G,\chi_{\cB}(x)\rangle
    \sim 
    \mathcal{N}(0,q^d)$ as this linear Gaussian combination has variance
    \begin{equation*}
    \mathrm{Var}(\langle G,\chi_\cB(x) \rangle =\chi_\cB(x)^\top \mathrm{Cov}(G) \chi_\cB(x)=\chi_\cB(x)^\top I_b \chi_\cB(x)=\Vert \chi_\cB(x)\Vert_2^2=q^d.
    \end{equation*}
    Hence, for every $t>0$,
\begin{equation*}
    \prob\left(
        |\langle G,\chi_{\cB}(x)\rangle|>t
    \right)
    \le
    2\exp\left(
        -\frac{t^2}{2q^d}
    \right).
\end{equation*}

Only the first $dq$ coordinates are active, so there are $2^{dq}$ relevant
block sign tuples. 
Set $K_d\coloneqq 4\sqrt{d+1}$ and take $t
    =
    K_d q^{d/2}\sqrt{q+\log(1/\delta)}$.
A union bound gives
\begin{align*}
    \prob\left(
        \sup_{x\in\cuben}
        |\langle G,\chi_{\cB}(x)\rangle|
        >
        K_d q^{d/2}\sqrt{q+\log(1/\delta)}
    \right)
    &\le
    2^{dq+1}
    \exp\left(
        -\frac{K_d^2}{2}
        \bigl(q+\log(1/\delta)\bigr)
    \right)  \\
    &\le
    \frac{\delta}{16},
\end{align*}
for every $q\ge 1$ and every $0<\delta\le 1/4$, by the choice of $K_d$.

On the complementary event, $\norm{f_{\mu G}}_{L_\infty}
    \le
    \mu K_d q^{d/2}\sqrt{q+\log(1/\delta)} $.
Using the definition of $\mu$ in \Cref{eq:gaussian-prior}, this becomes $\norm{f_{\mu G}}_{L_\infty}
    \le
    C_0K_dB_d\eps $.
Thus, if
\begin{equation*}
    \eps
    \le
    \frac{1}{C_0K_dB_d},
\end{equation*}
then $\norm{f_{\mu G}}_{L_\infty}\le 1$ on this event. Equivalently, since
$K_d=4\sqrt{d+1}$, this condition is implied by $\eps
    \le
    \frac{c_b}{B_d\sqrt d}$
for a sufficiently small universal constant $c_b>0$, after adjusting constants.
Therefore,
\begin{equation*}
    \prob\left(
        \norm{f_{\mu G}}_{L_\infty}\le 1
    \right)
    \ge
    1-\frac{\delta}{16}.\qedhere
\end{equation*}
\end{proof}

\subsection{Proof of \Cref{lem:walsh-design}} \label{app:walsh-design}

\begin{proof}
Let $X\sim\Unif(\cuben)$ and write $\varphi(X)=\chi_{\le d}(X)\in\reals^D$.
By orthogonality of the Walsh basis,
$\expectation[\varphi(X)\varphi(X)^\transpose]=I_D$. Moreover, since each
Walsh function takes values in $\cube$, $\norm{\varphi(X)}_2^2=D$ almost surely.
Let $X_1,\ldots,X_m$ be independent copies of $X$ and set
\begin{equation*}
    G=\frac{1}{m}\sum_{t=1}^m
    \varphi(X_t)\varphi(X_t)^\transpose .
\end{equation*}
Hence, $\expectation[G]=I_D$.

We apply \Cref{lem:matrix_bernstein} to the
centered matrices $Z_t=\varphi(X_t)\varphi(X_t)^\transpose-I_D$. These matrices
are independent, symmetric, and have a mean of zero. Write
$A_t=\varphi(X_t)\varphi(X_t)^\transpose$. Since $A_t$ is rank one and
$\norm{\varphi(X_t)}_2^2=D$, its eigenvalues are $D,0,\ldots,0$. Hence, the
eigenvalues of $Z_t=A_t-I_D$ are $D-1,-1,\ldots,-1$, and therefore
$\norm{Z_t}_{\mathrm{op}}\le D$.

Furthermore, $A_t^2=\norm{\varphi(X_t)}_2^2 A_t=D A_t$. Since
$\expectation[A_t]=I_D$, it follows that
\begin{equation*}
    \expectation[Z_t^2]
    =
    \expectation[(A_t-I_D)^2]
    =
    \expectation[A_t^2]-2\expectation[A_t]+I_D
    =
    (D-1)I_D .
\end{equation*}
Thus, the Bernstein variance proxy satisfies
\begin{equation*}
    \sigma^2
    =
    \norm{\sum_{t=1}^m\expectation[Z_t^2]}_{\mathrm{op}}
    =
    m(D-1)
    \le mD.
\end{equation*}
Since $G-I_D=m^{-1}\sum_{t=1}^m Z_t$, matrix Bernstein gives, for a universal
constant $c>0$,
\begin{align*}
    \prob\left(
        \norm{G-I_D}_{\mathrm{op}}>\frac{1}{2}
    \right)
    &=
    \prob\left(
        \norm{\sum_{t=1}^m Z_t}_{\mathrm{op}}>\frac{m}{2}
    \right) \\
    &\le
    2D\exp\left(
        -c\min\left\{
            \frac{m^2}{mD},
            \frac{m}{D}
        \right\}
    \right) \\
    &\le
    2D\exp\left(-c\frac{m}{D}\right).
\end{align*}

Choosing $m\ge C D\log(2D)$ with $C>0$ sufficiently large makes the right-hand side strictly smaller than $1$. Hence, with positive probability,
\begin{equation*}
    \norm{G-I_D}_{\mathrm{op}}\le \frac{1}{2}.
\end{equation*}
Choosing one realization in this event gives cube points
$x_1,\ldots,x_m\in\cuben$ such that
\begin{equation*}
    \frac12 I_D
    \preceq
    \frac1m\sum_{t=1}^m
    \chi_{\le d}(x_t)\chi_{\le d}(x_t)^\transpose
    \preceq
    \frac32 I_D .
\end{equation*}
In particular, the lower bound in the statement holds.
\end{proof}

\subsection{Proof of the upper bound in \Cref{thm:low-degree-main}} \label{app:low-degree-main-upperbound}

\begin{proof}
Fix a design satisfying \Cref{lem:walsh-design}, and let
$\Phi\in\R^{m\times D}$ be the matrix whose $t$-th row is
$\chi_{\le d}(x_t)^\top$. By the choice of the design,
\begin{equation*}
    \frac1m\Phi^\top\Phi
    =
    \frac1m\sum_{t=1}^m
    \chi_{\le d}(x_t)\chi_{\le d}(x_t)^\top
    \succeq \frac12 I_D .
\end{equation*}
In particular, $\Phi^\top\Phi$ is invertible, and $\left(\frac1m\Phi^\top\Phi\right)^{-1}\preceq 2I_D$.

We query $Y_t=f(x_t)+\xi_t$ for $t \in [m]$, and compute the least-squares estimator
\begin{equation*}
    \widehat\theta
    =
    (\Phi^\top\Phi)^{-1}\Phi^\top Y.
\end{equation*}
Writing $\theta=\theta(f)$, we have $Y=\Phi\theta+\xi$, where
$\xi=(\xi_1,\ldots,\xi_m)^\top$. Therefore,
\begin{equation*}
    \widehat\theta-\theta
    =
    (\Phi^\top\Phi)^{-1}\Phi^\top(\Phi\theta+\xi)-\theta 
    =
    (\Phi^\top\Phi)^{-1}\Phi^\top\xi .
\end{equation*}
Thus, for every fixed $x\in\cuben$,
\begin{equation*}
    \widehat f(x)-f(x)
    =
    \chi_{\le d}(x)^\top(\widehat\theta-\theta)  
    =
    \chi_{\le d}(x)^\top
    (\Phi^\top\Phi)^{-1}\Phi^\top\xi .
\end{equation*}
Define $a(x)
    \coloneqq
    \Phi(\Phi^\top\Phi)^{-1}\chi_{\le d}(x)
    \in\R^m$ so that
\begin{equation*}
    \widehat f(x)-f(x)
    =
    a(x)^\top \xi
    =
    \sum_{t=1}^m a_t(x)\xi_t .
\end{equation*}

We now bound the Euclidean norm of $a(x)$. Since $a(x)
    =
    \Phi(\Phi^\top\Phi)^{-1}\chi_{\le d}(x)$,
we get
\begin{align*}
    \norm{a(x)}_2^2
    &=
    \chi_{\le d}(x)^\top
    (\Phi^\top\Phi)^{-1}
    \Phi^\top\Phi
    (\Phi^\top\Phi)^{-1}
    \chi_{\le d}(x)  \\
    &=
    \chi_{\le d}(x)^\top
    (\Phi^\top\Phi)^{-1}
    \chi_{\le d}(x)  \\
    &=
    \frac1m
    \chi_{\le d}(x)^\top
    \left(\frac1m\Phi^\top\Phi\right)^{-1}
    \chi_{\le d}(x).
\end{align*}
Using $\left(\frac1m\Phi^\top\Phi\right)^{-1}\preceq 2I_D$ and the orthonormality of the Walsh basis, we obtain
\begin{align*}
    \norm{a(x)}_2^2
    &\le
    \frac2m\norm{\chi_{\le d}(x)}_2^2=\frac{2D}{m} .
\end{align*}

Since the noises $\xi_1,\ldots,\xi_m$ are independent and
$\sigma$-subgaussian, the linear combination $\sum_{t=1}^m a_t(x)\xi_t$
 is $\sigma\norm{a(x)}_2$-subgaussian. Consequently, for every fixed
$x\in\cuben$, using the previous inequality on $\Vert a(x) \Vert_2^2$, this gives
\begin{equation*}
    \prob_f\left(
        |\widehat f(x)-f(x)|>\eps
    \right)
    \le
    2\exp\left(
        -\frac{\eps^2}{2\sigma^2\norm{a(x)}_2^2}
    \right) \le
    2\exp\left(
        -\frac{m\eps^2}{4\sigma^2D}
    \right).
\end{equation*}

Taking a union bound over $\cuben$, 
\begin{equation*}
    \prob_f\left(
        \norm{\widehat f-f}_{L_\infty}>\eps
    \right)
    \le
    2^{n+1}
    \exp\left(
        -\frac{m\eps^2}{4\sigma^2D}
    \right).
\end{equation*}
Therefore, the right-hand side is at most $\delta$ as soon as $2^{n+1}
    \exp\left(
        -\frac{m\eps^2}{4\sigma^2D}
    \right)
    \le \delta$, or equivalently it is sufficient that  $\frac{m\eps^2}{4\sigma^2D}
    \ge
    (n+1)\log 2+\log(1/\delta)$.
Thus, for a sufficiently large universal constant $C>0$, it suffices to take $m
    \ge
    C\frac{\sigma^2D}{\eps^2}
    \left(n+\log(1/\delta)\right)$.
Combining this with the sample size condition for a well conditioned design proves the desired upper bound.
\end{proof}


\subsection{Proof of \Cref{lem:posterior-trace-main}} \label{app:posterior-trace-main}

\begin{proof}
The formula for the covariance after conditioning on the realized adaptive design comes from the self-conjugacy of the Gaussian distribution, see \Cref{lem:posterior-cov} for details.  Since $\Sigma_m^{-1}=I+
  \frac{\mu^2}{\sigma^2}
  \sum_{t=1}^m
  \chi_{\cB}(X_t)\chi_{\cB}(X_t)^\top\succeq I$, we have $0\preceq\Sigma_m\preceq I$. Also,
$\norm{\chi_{\cB}(X_t)}_2^2=b=q^d$ for every query, hence
\[
  \tr(\Sigma_m^{-1})
  =
  b+
  \frac{\mu^2}{\sigma^2}
  \sum\nolimits_{t=1}^m\norm{\chi_{\cB}(X_t)}_2^2
  =
  b\big(1+\frac{\mu^2m}{\sigma^2}\big) \leq 2 b,
\]
where the last inequality comes from the constraint on $m$. 
Now for $\lambda_1,\dots,\lambda_{b}$ the eigenvalues of $\Sigma_m$, we have by Cauchy--Schwarz
\begin{equation*}
\tr(\Sigma_m) \cdot \tr(\Sigma_m^{-1})=  \sum\nolimits_{i=1}^b \lambda_i \cdot \sum\nolimits_{i=1}^b 1/\lambda_i \geq b^2.
\end{equation*}
Hence, $\tr(\Sigma_m) \geq b^2/\tr(\Sigma_m^{-1}) \geq b/2=q^{d}/2$.
\end{proof}

\subsection{Proof of \Cref{lem:width-main}} \label{app:width-main}

\begin{proof}[Proof of \Cref{lem:width-main}]
Recall that \(b=q^d\), and identify the tensor \(W\) with a vector in \(\R^b\). Let $A=\{j:\operatorname{Var}(W_j)\ge 1/4\}\subset [b]$ be the set of coordinates whose variances are at least \(1/4\). Since
\(0\preceq \Sigma\preceq 2I\), we have \(0\le \operatorname{Var}(W_j)\le 2\) for all \(j\), and hence
\begin{equation*}
b \leq \tr(\Sigma)= \sum_{j \in [b]} \operatorname{Var}(W_j)
= \sum_{j \in A} \operatorname{Var}(W_j) + \sum_{j \not \in A} \operatorname{Var}(W_j)
\leq 2 \vert A \vert + \frac{1}{4}(b-\vert A\vert )
= \frac{7}{4}\vert A \vert + \frac{1}{4}b .
\end{equation*}
Rearranging yields \(\vert A \vert \geq 3b/7\). Let \(r=2d/(d+1)\). 

Applying Hölder's inequality to the restriction of \(W\) to
the index set \(A\), we get $\Vert W_A \Vert_1 \leq \vert A \vert^{1-\frac{1}{r}}\Vert W_A \Vert_r$. Since removing coefficients can only decrease the \(\| \cdot \|_r\)-norm, 
\begin{equation*}
 \Vert W \Vert_r \geq \Vert W_A \Vert_r
  \geq \vert A \vert^{\frac{1}{r}-1}\Vert W_A \Vert_1 .
\end{equation*}
The expectation of the absolute value of a centered Gaussian random variable
with variance \(\operatorname{Var}(W_j)\) is $\expectation[\vert W_j\vert]=\sqrt{\operatorname{Var}(W_j)}\sqrt{2/\pi}$.
Therefore, by linearity of expectation and the definition of \(A\),
\begin{equation*}
  \expectation[\Vert W_A \Vert_1]
  =\sum\nolimits_{j\in A}\expectation[\vert W_j\vert]
  \geq \vert A\vert \cdot \frac{1}{2}\sqrt{2/\pi}
  = \frac{1}{\sqrt{2\pi}}\vert A\vert .
\end{equation*}
Combining the previous inequalities, and using that $(3/7)^{\frac1r}\geq 3/7$ as $r\geq 1$, there is a universal constant $c_0>0$ such that
\begin{equation*}
\expectation[\Vert W \Vert_r]
\geq
\vert A \vert^{\frac1r-1}\expectation[\Vert W_A \Vert_1]
\geq
\frac{1}{\sqrt{2\pi}}\vert A\vert^{\frac1r}
\geq
\frac{1}{\sqrt{2\pi}}\big(\frac{3}{7}b\big)^{\frac1r}
\geq c_0 b^{\frac1r}.
\end{equation*}

Write \(W=\Sigma^{\frac{1}{2}}g\) for \(g\sim\mathcal{N}(0,I_b)\). Recall that for
\(1\le r\le 2\), we have the norm comparison inequality $\Vert z \Vert_r \leq b^{\frac1r-\frac{1}{2}}\Vert z \Vert_2$ for all $z \in \reals^b$.  Indeed, For \(z \in \reals^b\), by Hölder,
\[
\sum_{i=1}^b |z_i|^r
\le
\left(\sum_{i=1}^b |z_i|^{2}\right)^{r/2}
\left(\sum_{i=1}^b 1\right)^{1-r/2}
=
\norm{z}_2^r b^{1-r/2}.
\]
Taking \(r\)-th roots gives the claim.
Thus, for any \(g,g'\in\reals^b\),
\begin{equation*}
\Vert \Sigma^{\frac12}(g-g')\Vert_r
\leq
b^{\frac1r-\frac12}\Vert \Sigma^{\frac12}(g-g')\Vert_2
\leq
\sqrt{2}\,b^{\frac1r-\frac12}\Vert g-g'\Vert_2,
\end{equation*}
where the last inequality uses \(\Sigma\preceq 2I_b\). Hence,
\(g\mapsto \Vert \Sigma^{\frac12}g\Vert_r\) is
\(\sqrt{2}\,b^{\frac1r-\frac12}\)-Lipschitz.  Gaussian concentration for Lipschitz
functions gives, for every \(t>0\),
\[
\prob\left(
  \Vert W \Vert_r
  \leq
  \expectation[\Vert W \Vert_r]-t
\right)
\leq
\exp\left(
  -\frac{t^2}{4b^{\frac2r-1}}
\right).
\]
Taking \(t=\expectation[\Vert W \Vert_r]/2\) and using
\(\expectation[\Vert W \Vert_r]\geq c_0b^{\frac1r}\), we obtain $\prob\left(
  \Vert W \Vert_r
  \leq
  \frac12\expectation[\Vert W \Vert_r]
\right)\leq \exp(-c_1 b)$ for another universal constant \(c_1>0\). For $b$ large enough, this will yield a probability of less than $1/4$. Decreasing the universal constant if
necessary to handle the finitely many small values of \(b\), which remains strictly positive as a minimum over a finite set of strictly positive values, this implies that
there exists a universal constant \(c>0\) such that $\prob\left(\Vert W \Vert_r \geq c b^{\frac1r}\right)\geq \frac34$.

Finally, by the degree-$d$ Bohnenblust--Hille inequality applied to the block-multilinear
polynomial associated with \(W\), $\Vert W \Vert_r \leq B_d \Vert f_W \Vert_{L_{\infty}}
  = B_d \Vert W \Vert_{\mathrm{inj}}$. Therefore,
  \begin{equation*}
\prob(
    B_d\Vert W \Vert_{\mathrm{inj}} \geq c b^{\frac1r}
  )
  \geq
  \prob(\Vert W \Vert_r \geq c b^{\frac1r})
  \geq \frac34.
  \end{equation*}
  Since \(b^{\frac1r}=(q^d)^{(d+1)/(2d)}=q^{(d+1)/2}\), the claim follows. \qedhere

\end{proof}


\subsection{Proof of \Cref{prop:diameter-main}} \label{app:diameter-main}

\begin{proof}

Conditionally on $\mathcal{H}_m$, draw $G'$ independently from the same
posterior distribution as $G$. Conditionally on the observations, $G-G'$ is a
centered Gaussian with covariance $2\Sigma_m$. By
\Cref{lem:posterior-trace-main}, we have $0\preceq 2\Sigma_m\preceq 2I$ and
$\tr(2\Sigma_m)\ge q^d$. Applying the previous lemma conditionally on
$\mathcal{H}_m$ gives
\begin{equation*}
\prob(\Vert G-G' \Vert_{\mathrm{inj}}\ge c q^{(d+1)/2}/B_d
\,\mid\,\mathcal{H}_m)\ge 3/4.    
\end{equation*}

This coefficient-tensor separation is exactly a scaled $L_\infty$ separation of the corresponding functions $\Vert f_{\mu G}-f_{\mu G'} \Vert_{L_\infty}=\mu \Vert G-G' \Vert_{\mathrm{inj}}$. 
Using the definition of $\mu$ and the assumption $\log(1/\delta)\le q$,
\begin{equation}
  \mu q^{(d+1)/2}
  =
  C_0 B_d\eps\sqrt{\frac{q}{q+\log(1/\delta)}}
  \ge
  \frac{C_0B_d}{\sqrt2}\eps.
\end{equation}
Therefore, on the large-injective-norm event, $\norm{f_{\mu G}-f_{\mu G'}}_{L_\infty}
  =
  \mu\norm{G-G'}_{\mathrm{inj}}
  \ge
  cC_0\eps/\sqrt2$.
Thus, choosing $C_0>0$ sufficiently large, two independent posterior draws are separated by more
than $2\eps$ in uniform norm with conditional probability at least $3/4$.  On
that event, a fixed estimator $\smash[t]{\widehat f}$ cannot be within $\eps$ of both
posterior draws. Whenever $\norm{f_{\mu G}-f_{\mu G'}}_{L_\infty}>2\eps$, at least one of the two inequalities $\Vert f_{\mu G}- \widehat f \Vert_{L_\infty}>\eps$ and $\Vert f_{\mu G'}- \widehat f \Vert_{L_\infty}>\eps$ must hold, as otherwise by triangle inequality we would have 
\begin{equation*}
2\eps <\norm{f_{\mu G}-f_{\mu G'}}_{L_\infty} \leq \Vert f_{\mu G}- \widehat f \Vert_{L_\infty}+ \Vert f_{\mu G'}- \widehat f \Vert_{L_\infty} \leq 2 \eps    
\end{equation*}
Taking conditional posterior probabilities and using that $f_{\mu G}$ and $f_{\mu G'}$ have the same conditional law, gives by symmetry
\begin{align*}
\frac{3}{4} &\leq \prob(\Vert f_{\mu G} -f_{\mu G'}\Vert_{L_\infty}>2\eps\mid\mathcal{H}_m)\\
  &\le
  \prob(\Vert f_{\mu G} -\widehat f\Vert_{L_\infty}>\eps\mid\mathcal{H}_m)
  +
\prob(\Vert f_{\mu G'} -\widehat f\Vert_{L_\infty}>\eps\mid\mathcal{H}_m)
  \\
  &=2\prob(\Vert f_{\mu G}-\widehat  f\Vert_{L_\infty}>\eps\mid\mathcal{H}_m).     \qedhere
\end{align*}


\end{proof}

\subsection{Proof of \Cref{prop:one-point-main}} \label{app:one-point-main}

\begin{proof}
By \Cref{lem:posterior-trace-main}, $\tr(\Sigma_m)\ge q^d/2$.  Averaging the
posterior variance of the prediction over the active cube reveals a point where
this variance is still large.  Indeed, the block-multilinear Walsh features are
orthonormal under the uniform measure on the active cube, so
$\expectation_X[\chi_{\cB}(X)\chi_{\cB}(X)^\top]=I$. Let $X \sim \cuben$. Using that the trace of a scalar is equal to the scalar, we obtain that
\begin{align*}
\expectation[\chi_{\cB}(X)^\top\Sigma_m \chi_{\cB}(X)]
&=
\expectation[\tr\!\left(\chi_{\cB}(X)^\top\Sigma_m \chi_{\cB}(X)\right) ] \\
&=
\expectation[\tr\!\left(\Sigma_m\chi_{\cB}(X)\chi_{\cB}(X)^\top\right) ]
\\
&=
\tr\!\left(\Sigma_m\expectation[\chi_{\cB}(X)\chi_{\cB}(X)^\top]\right) \\
&=
\tr(\Sigma_m) \geq q^d/2.
\end{align*}
Thus, there exists at least one $x^\star$ such that
$\chi_{\cB}(x^\star)^\top\Sigma_m \chi_{\cB}(x^\star)\ge q^d/2$.
Conditional on the observations, the scalar $f_{G}(x^\star)=\langle G, \chi_{\cB}(x^\star) \rangle$ is Gaussian.
Its posterior variance is at least
\begin{equation*}
   \mu^2 \chi_{\cB}(x^\star)^\top\Sigma_m \chi_{\cB}(x^\star)
  \ge
  \mu^2 q^d/2
  =
  \frac{C_0^2B_d^2\eps^2}{2(q+\log(1/\delta))}
  \ge
  \frac{C_0^2B_d^2\eps^2}{4\log(1/\delta)},   
\end{equation*}
where the last inequality uses $\log(1/\delta)>q$.

Let $s_\star^2$ denote this variance.  The estimator $\smash[t]{\widehat f}$ is fixed after
the transcript is fixed.  Among all intervals of radius $\eps$, the Gaussian
posterior mass is maximized by centering the interval at the posterior mean.
Therefore, regardless of the value $\widehat f(x^\star)$, for some $Z\sim\N(0,1)$:
\begin{equation*}
    \prob\big (
  |f_{\mu G}(x^\star)-\widehat f(x^\star)|>\eps
  \,\mid \,\mathcal{H}_m
  \big)
  \ge
  \prob\big(|Z|>\frac{\eps}{s_\star}\big)
  \ge
  \prob\big(|Z|>\frac{2\sqrt{\log(1/\delta)}}{C_0B_d}\big),  
\end{equation*}
Choosing $C_0$ large enough and applying
\Cref{lem:gauss-log-tail} gives the lower bound $2e^{-\log(1/\delta)}=2\delta$.  Since an
error at the single point $x^\star$ is smaller than the  $L_\infty$ error, the
same lower bound holds for $\Vert f_{G}-\widehat f \Vert_{L_\infty}$. \qedhere
\end{proof}

\subsection{Proof of the lower bound of \Cref{thm:low-degree-main}} \label{app:low-degree-main} 

\begin{proof}
Choose $C_0>0$ sufficiently large so that \Cref{prop:diameter-main,prop:one-point-main} 
apply. By the assumptions, for $c>0$ small enough,
$m \le c\sigma^2q^d(q+\log(1/\delta))/(B_d^2\eps^2)$, so
\Cref{lem:posterior-trace-main} applies for every transcript $\mathcal H_m$. If $\log(1/\delta)\le q$, \Cref{prop:diameter-main} gives conditional posterior
error probability at least $3/8$, otherwise, \Cref{prop:one-point-main} gives
conditional posterior error probability at least $2\delta$. Hence, under the
unconditioned Gaussian prior,
$\prob(\Vert f_{\mu G}-\smash[t]{\widehat f}\Vert_{L_\infty}>\eps)
\ge \min(3/8, 2\delta)$.

Conditioning on the bounded event loses at most $\delta/16$ in Bayes risk by
\Cref{lem:prior-conditioning}. Since $\delta\le1/4$, the conditioned prior has Bayes
error probability at least $\min(3/8, 2\delta)-\delta/16 \ge \delta$. Since this
prior is supported on $\cP_{n,d}$, the minimax risk over $\cP_{n,d}$ is at least
$\delta$. Finally, $q=\lfloor n/d\rfloor\ge n/(2d)$ gives $q^d\ge D/(2e)^d$, and
$q+\log(1/\delta)\ge(n+\log(1/\delta))/(2d)$. This concludes the proof.
\end{proof}

\section{Omitted proofs of \Cref{sec:sparse}}

\subsection{Proof of \Cref{lem:count-gs}} \label{app:count-gs}

\begin{proof}
There are exactly \(2^n\) Walsh basis functions on \(\cuben\). To construct a function
\(g\in\mathcal G_s\), we first choose a support of size \(k\le s\), and then choose
one sign in \(\cube\) for each selected character. Hence,
\begin{equation*}
    |\mathcal G_s|
    \le
    \sum_{k=0}^{s} \binom{2^n}{k}2^k .
\end{equation*}
Using the standard estimate \(\binom{N}{k}\le (eN/k)^k\), with \(N=2^n\), we get
\begin{equation*}
    \binom{2^n}{k}2^k
    \le
    \left(\frac{2e\,2^n}{k}\right)^k .
\end{equation*}
Therefore,
\begin{equation*}
    |\mathcal G_s|
    \le
    (s+1)\left(\frac{2e\,2^n}{s}\right)^s .
\end{equation*}
Taking logarithms yields
\begin{equation*}
    \log |\mathcal G_s|
    \le
    \log(s+1)+s\log(2e)+sn\log 2-s\log s .
\end{equation*}
Since \(s\le n\), the positive terms are bounded by \(Csn\) for a universal
constant \(C>0\). Thus, $|\mathcal G_s|\le \exp(Csn)$ as claimed. 
\end{proof}

\subsection{Proof of \Cref{lem:top-s-thresholding}} \label{app:top-s-thresholding}

\begin{proof}
Let $T=\supp(\theta)$ and $\widehat T=\supp(\smash[t]{\widehat \theta})$. Since
$\theta$ is $s$-sparse and $\smash[t]{\widehat \theta}$ keeps at most $s$
coordinates, we have $|T|\le s$ and $|\widehat T|\le s$. We decompose
\begin{align*}
    \norm{\smash[t]{\widehat \theta}-\theta}_1
    &=
    \sum_{S\in T\cap \widehat T}
    |\widehat \theta_S-\theta_S|
    +
    \sum_{S\in \widehat T\setminus T}
    |\widehat \theta_S-\theta_S|
    +
    \sum_{S\in T\setminus \widehat T}
    |\widehat \theta_S-\theta_S| .
\end{align*}
Writing $e=\smash{\widetilde \theta}-\theta$, we have
$\widehat\theta_S=\widetilde\theta_S$ on $\widehat T$. Hence,
\begin{equation*}
    \norm{\smash[t]{\widehat \theta}-\theta}_1
    =
    \sum_{S\in T\cap \widehat T}|e_S|
    +
    \sum_{S\in \widehat T\setminus T}|e_S|
    +
    \sum_{S\in T\setminus \widehat T}|\theta_S| .
\end{equation*}

Since $\widehat T$ contains the $s$ largest empirical magnitudes and
$|T|\le s$, every omitted index in $T\setminus \widehat T$ can be matched
injectively to an index in $\widehat T\setminus T$ with at least as large
empirical magnitude, as $\vert T \setminus \widehat T \vert =\vert \widehat  T \setminus T \vert$. Therefore,
\begin{equation*}
    \sum_{S\in T\setminus \widehat T}
    |\widetilde \theta_S|
    \le
    \sum_{S\in \widehat T\setminus T}
    |\widetilde \theta_S| .
\end{equation*}
On $\widehat T\setminus T$, we have $\theta_S=0$, and thus
$\widetilde\theta_S=e_S$. Moreover, for $S\in T\setminus\widehat T$,
$|\theta_S|\le |\widetilde\theta_S|+|e_S|$. Hence,
\begin{equation*}
    \sum_{S\in T\setminus \widehat T}|\theta_S|
    \le
    \sum_{S\in T\setminus \widehat T}
    |\widetilde \theta_S|
    +
    \sum_{S\in T\setminus \widehat T}|e_S|  
    \le
    \sum_{S\in \widehat T\setminus T}|e_S|
    +
    \sum_{S\in T}|e_S| .
\end{equation*}
Substituting this bound into the previous decomposition gives
\begin{equation*}
    \norm{\smash[t]{\widehat \theta}-\theta}_1
    \le
    \sum_{S\in T\cap \widehat T}|e_S|
    +
    2\sum_{S\in \widehat T\setminus T}|e_S|
    +
    \sum_{S\in T}|e_S| .
\end{equation*}
Each of the sets appearing in these sums has cardinality at most $s$, so each
sum is bounded by $R_s(e)$. Therefore,
\begin{equation*}
    \norm{\smash[t]{\widehat \theta}-\theta}_1
    \le
    4 \Vert \widetilde \theta -\theta \Vert_{(s)}. \qedhere
\end{equation*}
\end{proof}

\subsection{Proof of upper bound of \Cref{thm:sparse-noisy}} 
\label{app:sparse-noisy-upperbound}

\begin{proof}
Let $e\coloneqq \smash{\widetilde\theta}-\theta$. Fix a signed sparse Walsh polynomial $g =\sum_{S\in \cA}\eta_S\chi_S \in \cG_s $ with $\qquad|\cA|\le s$. 
Then
\begin{align*}
    \sum_{S\in \cA}\eta_S e_S
    =
    \sum_{S\in \cA}\eta_S
    \left(
        \frac1m\sum_{t=1}^m Y_t\chi_S(X_t)-\theta_S
    \right)  
    =
    \frac1m\sum_{t=1}^m Y_t g(X_t)
    -
    \sum_{S\in \cA}\eta_S\theta_S .
\end{align*}
By orthogonality of the Walsh basis, $\sum_{S\in \cA}\eta_S\theta_S
=\expectation[f(X)g(X)]$. Hence,
$\sum_{S\in \cA}\eta_S e_S=R_{\mathrm{des}}+R_{\mathrm{noise}}$, where
\begin{equation*}
    R_{\mathrm{des}}
    \coloneqq
    \frac1m\sum_{t=1}^m
    \left(
        f(X_t)g(X_t)-\expectation[f(X)g(X)]
    \right),
    \qquad
    R_{\mathrm{noise}}
    \coloneqq
    \frac1m\sum_{t=1}^m \xi_t g(X_t).
\end{equation*}
We control these two terms uniformly over $\cG_s$.

First consider the design term. Set
$Z_t\coloneqq f(X_t)g(X_t)-\expectation[f(X)g(X)]$. Since
$\norm{f}_{L_\infty}\le1$, and since \Cref{lem:signed-moments} gives
$\expectation[g(X)^2]\le s$ and $\norm{g}_{L_\infty}\le s$, we have
$\expectation[Z_t^2]\le \expectation[f(X_t)^2g(X_t)^2]\le s$. Moreover,
\begin{equation*}
    |Z_t|
    \le
    |g(X_t)|+\expectation[|g(X)|]
    \le
    s+\sqrt{\expectation[g(X)^2]}
    \le
    2s .
\end{equation*}
Bernstein's inequality therefore gives, for $0<r\le1$,
\begin{align*}
    \prob(|R_{\mathrm{des}}|>r)
    \le
    2\exp\left(
        -c\min\left\{
            \frac{m^2r^2}{ms},
            \frac{mr}{2s}
        \right\}
    \right)  
    \le
    2\exp\left(
        -c\frac{mr^2}{s}
    \right).
\end{align*}
By a union bound over $\cG_s$ and \Cref{lem:count-gs}, we have
$\sup_{g\in\cG_s}|R_{\mathrm{des}}|\le \eps/16$ with probability at least $1-\delta/3$,
provided
\begin{equation*}
    m
    \ge
    C\frac{s}{\eps^2}
    \left(
        sn+\log\left(\frac1\delta\right)
    \right).
\end{equation*}

We next control the empirical second moments. For fixed $g\in\cG_s$, let
$W_t\coloneqq g(X_t)^2-\expectation[g(X)^2]$. By \Cref{lem:signed-moments},
$\expectation[W_t^2]\le \expectation[g(X_t)^4]\le s^3$, and
\begin{equation*}
    |W_t|
    \le
    |g(X_t)|^2+\expectation[g(X)^2]
    \le
    s^2+s
    \le
    2s^2 .
\end{equation*}
Another application of Bernstein's inequality gives
\begin{align*}
    \prob\left(
        \left|
            \frac1m\sum_{t=1}^m W_t
        \right|>2s
    \right)
    \le
    2\exp\left(
        -c\min\left\{
            \frac{4s^2m^2}{ms^3},
            \frac{2sm}{2s^2}
        \right\}
    \right)  
    \le
    2e^{-cm/s}.
\end{align*}
Thus, by a union bound over $\cG_s$, with probability at least $1-\delta/3$,
\begin{equation*}
    \frac1m\sum_{t=1}^m g(X_t)^2
    \le
    \expectation[g(X)^2]+2s
    \le
    3s
    \qquad
    \text{for all } g\in\cG_s,
\end{equation*}
provided $m
    \ge
    Cs\left(
        sn+\log\left(\frac1\delta\right)
    \right)$.
Since $\eps\le1$ and $\sigma^2+1\ge1$, this condition is absorbed by the final
sample-size requirement.

It remains to control the noise term. Conditional on $X_1,\ldots,X_m$,
\Cref{lem:subg-weighted} gives, for every fixed $g\in\cG_s$,
\begin{equation*}
    \prob\left(
        |R_{\mathrm{noise}}|>r
        \mid X_1,\ldots,X_m
    \right)
    \le
    2\exp\left(
        -
        \frac{m^2r^2}
        {2\sigma^2\sum_{t=1}^m g(X_t)^2}
    \right).
\end{equation*}
On the empirical second-moment event above,
$\sum_{t=1}^m g(X_t)^2\le 3ms$ for all $g\in\cG_s$. Hence, conditionally on
this event,
\begin{equation*}
    \prob\left(
        |R_{\mathrm{noise}}|>r
        \mid X_1,\ldots,X_m
    \right)
    \le
    2\exp\left(
        -c\frac{mr^2}{\sigma^2s}
    \right).
\end{equation*}
A union bound over $\cG_s$ gives $\sup_{g\in\cG_s}|R_{\mathrm{noise}}|\le \eps/16$ with
conditional probability at least $1-\delta/3$, provided
\begin{equation*}
    m
    \ge
    C\frac{\sigma^2s}{\eps^2}
    \left(
        sn+\log\left(\frac1\delta\right)
    \right).
\end{equation*}

On the intersection of the three good events, for every
$\cA\subseteq2^{[n]}$ with $|\cA|\le s$ and every sign vector
$(\eta_S)_{S\in\cA}\in\cube^{\cA}$,
\begin{equation*}
    \left|
        \sum_{S\in\cA}\eta_S e_S
    \right|
    =
    |R_{\mathrm{des}}+R_{\mathrm{noise}}|
    \le
    |R_{\mathrm{des}}|+|R_{\mathrm{noise}}|
    \le
    \frac{\eps}{8}.
\end{equation*}
Taking the supremum over $\cA$ and signs gives
$\norm{e}_{(s)}\le \eps/8\le \eps/4$. The total failure probability is at most
$\delta$.

Let $\smash[t]{\widehat\theta}$ be obtained by keeping the $s$ largest entries
of $\smash{\widetilde\theta}$ in absolute value and setting all other entries to
zero, and let $\smash[t]{\widehat f}$ be the Walsh polynomial associated with
$\smash[t]{\widehat\theta}$. On the event $\norm{e}_{(s)}\le\eps/4$,
\Cref{lem:top-s-thresholding} gives
\begin{align*}
    \norm{\smash[t]{\widehat f}-f}_{L_\infty}
    \le
    \sum_{S\subseteq[n]}
    |\widehat\theta_S-\theta_S|\norm{\chi_S}_{L_\infty}  
    \le
    \norm{\smash[t]{\widehat\theta}-\theta}_1
    \le
    4\norm{e}_{(s)}
    \le
    \eps .
\end{align*}
Combining all conditions on $m$, it suffices to take $m
    \ge
    C\frac{(\sigma^2+1)s}{\eps^2}
    \left(
        sn+\log\left(\frac1\delta\right)
    \right)$,
which is the desired upper bound.
\end{proof}

\subsection{Proof of \Cref{lem:systematic-family-main}} \label{app:systematic-family-main}

\begin{proof}
Each term $u_j\chi_{S_j}(v)$ is a Walsh function on the full cube, namely
$u_j\prod_{\ell\in S_j}v_\ell$. These $k$ Walsh functions are distinct because
they contain different first-block coordinates $u_j$. Hence, $F_M$ has at most
$k$ non-zero Walsh coefficients, and so $f_M=aF_M$ is $k$-sparse.

Moreover, since each Walsh function has magnitude $1$, we have
\begin{equation*}
    |F_M(u,v)|
    \le
    \sum_{j=1}^k |u_j\chi_{S_j}(v)|
    =
    k
\end{equation*}
for every $(u,v)\in\cube^k\times\cube^{n-k}$. Therefore,
$\norm{f_M}_{L_\infty}\le ak$.

We now prove the separation bound. Let
$J\coloneqq\{j\in[k]:S_j\ne T_j\}$, so that $|J|=d_{\mathrm H}(M,T)$. For
$j\in J$, the Walsh function $\chi_{S_j}\chi_{T_j}=\chi_{S_j\symmetricdifference T_j}$ is
nonconstant, since $S_j\symmetricdifference T_j\ne\emptyset$. Hence, if
$V\sim\Unif(\cube^{n-k})$, then $\chi_{S_j}(V)\chi_{T_j}(V)$ is uniform on
$\cube$, and in particular
\begin{equation*}
    \prob\left(\chi_{S_j}(V)\ne \chi_{T_j}(V)\right)
    =
    \frac12 .
\end{equation*}
It follows by linearity of expectation  that
\begin{equation*}
    \expectation_V\left[
        \sum_{j\in J}
        \indicator\{\chi_{S_j}(V)\ne \chi_{T_j}(V)\}
    \right] = \sum_{j \in J} \prob\left(\chi_{S_j}(V)\ne \chi_{T_j}(V)\right)
    =
    \frac{|J|}{2}.
\end{equation*}
Therefore, there exists $v^\star\in\cube^{n-k}$ such that
\begin{equation*}
    \sum_{j\in J}
    \indicator\{\chi_{S_j}(v^\star)\ne \chi_{T_j}(v^\star)\}
    \ge
    \frac{|J|}{2}.
\end{equation*}
Choose $u^\star\in\cube^k$ by setting
$u_j^\star\coloneqq \chi_{S_j}(v^\star)$ for every $j\in[k]$. Then
$u_j^\star\chi_{S_j}(v^\star)=1$ for every $j$, and hence
$F_M(u^\star,v^\star)=k$. On the other hand, the $j$-th term of $F_T$ at
$(u^\star,v^\star)$ is
$u_j^\star\chi_{T_j}(v^\star)
=
\chi_{S_j}(v^\star)\chi_{T_j}(v^\star)$, which is equal to $-1$ whenever
$\chi_{S_j}(v^\star)\ne\chi_{T_j}(v^\star)$. Thus,
\begin{align*}
    F_M(u^\star,v^\star)-F_T(u^\star,v^\star)
    =
    \sum_{j=1}^k
    \left(
        1-\chi_{S_j}(v^\star)\chi_{T_j}(v^\star)
    \right)  
    \ge
    2\sum_{j\in J}
    \indicator\{\chi_{S_j}(v^\star)\ne \chi_{T_j}(v^\star)\}  
    \ge
    |J|.
\end{align*}
Consequently,
$\norm{F_M-F_T}_{L_\infty}\ge |J|=d_{\mathrm H}(M,T)$. Multiplying by $a$
gives $\norm{f_M-f_T}_{L_\infty}\ge a d_{\mathrm H}(M,T)$.
\end{proof}

\subsection{Proof of \Cref{prop:support-search-main}} \label{app:support-search-main}

\begin{proof}
We use the family from \Cref{lem:systematic-family-main}. Let $k\le s$ be as
in the construction, and set $a\coloneqq 4\eps/k$. Then
$\norm{f_M}_{L_\infty}\le ak=4\eps\le 1$, provided $\eps\le 1/4$, and each
$f_M$ is $k$-sparse. Hence, $f_M\in\cS_{n,s}$.

Let $M=(S_1,\ldots,S_k)$ be uniformly distributed over the parameter set, and
let $\smash[t]{\widehat f}$ be any estimator. Define the projected estimator
\begin{equation*}
    \widehat M
    \in
    \argmin_N
    \norm{\smash[t]{\widehat f}-f_N}_{L_\infty}.
\end{equation*}
We will use Fano's lower bound to show that $\widehat M$ cannot be well estimated, which will prove the same for $\widehat f$. Fix $j\in[k]$ and condition on $M_{-j}=(S_\ell)_{\ell\ne j}$. The remaining
coordinate $S_j$ ranges over an alphabet of size $Q=2^{n-k}$. Take two possible
values $S,S'\subseteq[n-k]$ for this coordinate $j$, and let
\begin{equation*}
    M^S
    \coloneqq
    (S_1,\ldots,S_{j-1},S,S_{j+1},\ldots,S_k),
    \qquad
    M^{S'}
    \coloneqq
    (S_1,\ldots,S_{j-1},S',S_{j+1},\ldots,S_k).
\end{equation*}
Then $f_{M^S}$ and $f_{M^{S'}}$ differ only in their $j$-th Walsh term. More
precisely, for every $(u,v)\in\cube^k\times\cube^{n-k}$,
\begin{equation*}
    f_{M^S}(u,v)-f_{M^{S'}}(u,v)
    =
    a u_j\bigl(\chi_S(v)-\chi_{S'}(v)\bigr).
\end{equation*}
Since $u_j\in\cube$ and $\chi_S(v),\chi_{S'}(v)\in\cube$, this gives
$\norm{f_{M^S}-f_{M^{S'}}}_{L_\infty}\le 2a$. By \Cref{lem:adaptive-kl},
\begin{equation*}
    \operatorname{KL}(P_{M^S}^{(m)}\Vert P_{M^{S'}}^{(m)})
    \le
    \frac{m(2a)^2}{2\sigma^2}
    =
    \frac{2ma^2}{\sigma^2}.
\end{equation*}
Under the sample-size assumption $m\le c\sigma^2 k^2(n-k)/\eps^2$, and since
$a=4\eps/k$, this gives
\begin{equation*}
    \frac{2ma^2}{\sigma^2}
    \le
    32c(n-k).
\end{equation*}
Choosing $c>0$ sufficiently small, we can ensure that
$32c(n-k)\le (\log Q)/16=(n-k) \log(2)/16$. Thus, the average pairwise KL divergence, conditional
on $M_{-j}$, is at most $(\log Q)/16$. By \Cref{lem:fano},
\begin{equation*}
    \prob(\widehat S_j\ne S_j\mid M_{-j})
    \ge
    1-\frac{\frac{1}{16}\log Q+\log 2}{\log Q},
\end{equation*}
where $\widehat S_j$ denotes the $j$-th coordinate of $\widehat M$. For $n-k$
larger than a universal constant, the right-hand side is at least $3/4$. The
remaining finitely many cases are absorbed by decreasing $c$ if necessary.
Hence, for every $j\in[k]$,
\begin{equation*}
    \prob(\widehat S_j\ne S_j\mid M_{-j})
    \ge
    \frac34 .
\end{equation*}
Averaging over $M_{-j}$ and summing over $j\in[k]$, we obtain
\begin{equation*}
    \expectation[d_{\mathrm H}(M,\widehat M)]
    =
    \sum_{j=1}^k
    \prob(\widehat S_j\ne S_j)
    \ge
    \frac{3k}{4}.
\end{equation*}
Since $0\le d_{\mathrm H}(M,\widehat M)\le k$, we also have
\begin{align*}
    \expectation[d_{\mathrm H}(M,\widehat M)]
    &\le
    \frac{k}{2}\prob(d_{\mathrm H}(M,\widehat M)\le k/2)
    +
    k\prob(d_{\mathrm H}(M,\widehat M)> k/2)  \\
    &\le
    \frac{k}{2}
    +
    \frac{k}{2}\prob(d_{\mathrm H}(M,\widehat M)> k/2).
\end{align*}
Combining the two bounds yields
\begin{equation*}
    \prob(d_{\mathrm H}(M,\widehat M)> k/2)
    \ge
    \frac12 .
\end{equation*}

It remains to relate Hamming error to $L_\infty$ error. Suppose that
$\norm{\smash[t]{\widehat f}-f_M}_{L_\infty}\le\eps$. By the definition of
$\widehat M$, we have
$\norm{\smash[t]{\widehat f}-f_{\widehat M}}_{L_\infty}
\le
\norm{\smash[t]{\widehat f}-f_M}_{L_\infty}\le\eps$. Hence, by the triangle
inequality,
\begin{equation*}
    \norm{f_M-f_{\widehat M}}_{L_\infty}
    \le
    \norm{f_M-\smash[t]{\widehat f}}_{L_\infty}
    +
    \norm{\smash[t]{\widehat f}-f_{\widehat M}}_{L_\infty}
    \le
    2\eps
    =
    \frac{ak}{2}.
\end{equation*}
On the other hand, \Cref{lem:systematic-family-main} gives
$\norm{f_M-f_{\widehat M}}_{L_\infty}\ge a d_{\mathrm H}(M,\widehat M)$.
Therefore, the event $\norm{\smash[t]{\widehat f}-f_M}_{L_\infty}\le\eps$
implies $d_{\mathrm H}(M,\widehat M)\le k/2$. Equivalently,
\begin{equation*}
    \{d_{\mathrm H}(M,\widehat M)>k/2\}
    \subseteq
    \{\norm{\smash[t]{\widehat f}-f_M}_{L_\infty}>\eps\}.
\end{equation*}
Since the left-hand event has probability at least $1/2$, the same holds for
the right-hand event. Thus, the Bayes risk over this finite subfamily is at least
$1/2$, and consequently
\begin{equation*}
    \sup_{f\in\cS_{n,s}}
    \prob_f\left(
        \norm{\smash[t]{\widehat f}-f}_{L_\infty}>\eps
    \right)
    \ge
    \frac12 .
\end{equation*}
This proves the proposition.
\end{proof}

\subsection{Proof of \Cref{prop:sparse-highconf-main}} \label{app:sparse-highconf-main}

\begin{proof}
Consider the fixed-support subclass $ \mathcal H_s
    \coloneqq
    \{
        f_\theta(x)=\sum_{j=1}^s \theta_jx_j:
        \norm{\theta}_1\le1
    \}
    \subseteq
    \cS_{n,s}$.
On this subclass, $\norm{f_\theta-f_{\theta'}}_{L_\infty}
=\norm{\theta-\theta'}_1$, since the cube realizes every sign pattern on the first $s$ coordinates.

Remark that we also have $\mathcal{H}_s \subset \cP_{n,1}$. In fact, the proof is the same one-point posterior-variance argument as in
\Cref{prop:one-point-main}. Let $L\coloneqq\log(1/\delta)$, and put a Gaussian
prior $\theta=\mu G$, where $G\sim\N(0,I_s)$ and
\begin{equation*}
    \mu
    \coloneqq
    C_0\frac{\eps}{\sqrt{sL}},
\end{equation*}
for a sufficiently large universal constant $C_0>0$. We first work with the
unconditioned prior. Conditional on any realized transcript $\mathcal H_m$, the
posterior law of $G$ is Gaussian with covariance
\begin{equation*}
    \Sigma_m
    =
    \left(
        I_s
        +
        \frac{\mu^2}{\sigma^2}
        \sum_{t=1}^m
        X_{t,\le s}X_{t,\le s}^\top
    \right)^{-1},
\end{equation*}
where $X_{t,\le s}=(X_{t,1},\ldots,X_{t,s})$. Since
$\norm{X_{t,\le s}}_2^2=s$, if
$m\le c\sigma^2sL/\eps^2$, then, after taking $c>0$ sufficiently small,
the same trace argument as in \Cref{lem:posterior-trace-main} gives
$\tr(\Sigma_m)\ge s/2$.

Averaging over $\eta\sim\Unif(\cube^s)$ gives
$\expectation_\eta[\eta^\top\Sigma_m\eta]=\tr(\Sigma_m)$. Hence, there exists
$\eta^\star\in\cube^s$ such that
$\eta^{\star\top}\Sigma_m\eta^\star\ge s/2$. Therefore, conditional on
$\mathcal H_m$, the posterior variance of $f_{\mu G}(\eta^\star)$ is at least
$\mu^2s/2$. By the Gaussian tail bound used in \Cref{prop:one-point-main}, and
by choosing $C_0$ large enough, every estimator $\smash[t]{\widehat f}$ satisfies
\begin{equation*}
    \prob\left(
        \norm{f_{\mu G}-\smash[t]{\widehat f}}_{L_\infty}>\eps
        \mid \mathcal H_m
    \right)
    \ge
    2\delta .
\end{equation*}
Thus, the unconditioned Gaussian prior has Bayes risk at least $2\delta$.

It remains only to condition the prior so that it is supported on
$\mathcal H_s$. Let $\mathcal E\coloneqq\{\norm{\mu G}_1\le1\}$. Since
$L\ge s$, a standard chi-square tail bound, together with
$\norm{G}_1\le\sqrt{s}\norm{G}_2$, gives
$\prob(\mathcal E^c)\le\delta/2$ whenever $\eps\le\eps_0$, for a sufficiently
small universal constant $\eps_0>0$. Conditioning on $\mathcal E$ therefore
loses at most $\delta/2$ in Bayes risk. The conditioned prior is supported on
$\mathcal H_s\subseteq\cS_{n,s}$, and its Bayes risk is still at least $\delta$.
Consequently,
\begin{equation*}
    \sup_{\xi }\sup_{f\in\cS_{n,s}}
    \prob_f\left(
        \norm{f-\smash[t]{\widehat f}}_{L_\infty}>\eps
    \right)
    \ge
    \delta.\qedhere
\end{equation*}
\end{proof}

\end{document}